\documentclass[letterpaper, 11pt]{article}
\pdfoutput=1
\usepackage{graphicx,color}
\usepackage{graphics,latexsym,amsmath,amssymb,xspace,enumerate,paralist}
\usepackage{amsthm, verbatim}
\usepackage{wrapfig,subfigure,arydshln}
\usepackage[noend]{algorithmic}
\usepackage{algorithm}
\usepackage{geometry}
\newcommand{\scrD}{\ensuremath{\mathcal{D}}}
\newcommand{\scrE}{\ensuremath{\mathcal{E}}}

\newcommand{\scrG}{\ensuremath{\mathcal{G}}}

\newcommand{\scrP}{\ensuremath{\mathcal{P}}}

\newcommand{\fk}[1]{\ensuremath{f^{#1}_{k}}}
\newcommand{\f}[1]{\ensuremath{f^{#1}(\alpha)}}

\newcommand{\fb}[1]{\ensuremath{f^{#1}(\beta)}}

\definecolor{darkviolet}{RGB}{148,0,211}

\newtheorem{corollary}{Corollary}

\newtheorem{example}{Example}
\newtheorem{lemma}{Lemma}

\newtheorem{problem}{Problem}

\newtheorem{theorem}{Theorem}
\newtheorem{definition}{Definition}
\newtheorem{remark}{Remark}

\begin{document}

\title{Streaming Algorithms for Pattern Discovery over Dynamically Changing Event Sequences}
\author{
Debprakash Patnaik\thanks{Debprakash Patnaik is now with Amazon.com} and Naren Ramakrishnan\\
Virginia Tech, Blacksburg, VA\\
{ \{patnaik,naren\}@vt.edu}
\and
Srivatsan Laxman and Badrish Chandramouli\\
       {Microsoft Research}\\
	   { \{slaxman,badrishc\}@microsoft.com}
}

\maketitle
\begin{abstract}
Discovering frequent episodes over event sequences is an important data mining task. In many applications, events constituting the data sequence arrive as a stream, at furious rates, and recent trends (or frequent episodes) can change and drift due to the dynamical nature of the underlying event generation process. The ability to detect and track such the changing sets of frequent episodes can be valuable in many application scenarios. Current methods for frequent episode discovery are typically multipass algorithms, making them unsuitable in the streaming context. In this paper, we propose a new streaming algorithm for discovering frequent episodes over a window of recent events in the stream. Our algorithm processes events as they arrive, one batch at a time, while discovering the top frequent episodes over a window consisting of several batches in the immediate past. We derive approximation guarantees for our algorithm under the condition that frequent episodes are approximately well-separated from infrequent ones in every batch of the window. We present extensive experimental evaluations of our algorithm on both real and synthetic data. We also present empirical comparisons with baselines and adaptations of streaming algorithms from itemset mining literature.

\end{abstract}

\section{Introduction}
\label{sec:intro}

The problem of discovering interesting patterns from large datasets has been 
well studied in the form of
pattern classes such as itemsets,
sequential patterns, and episodes with temporal constraints. However, most of these
techniques
deal with static datasets, over which multiple passes are
performed. 

In many domains like telecommunication and computer security, it is becoming increasingly difficult to store 
and process data at speeds comparable to their generation rate. 
A few minutes of call logs data in a telecommunication network can easily run into millions of records.
Such data are referred to as \textit{data streams}~\cite{sketch19}. A \textit{data stream} is an unbounded sequence where new data points or events arrive continuously and often at very high rates. Many traditional data mining algorithms are rendered useless in this context as one cannot hope to store the entire data and then process it. Any method for data streams must thus
operate under the constraints of limited memory and processing time. In addition, 
the data must be processed faster than it is being generated. In this paper, we 
investigate the problem of mining temporal patterns called episodes under these constraints; while we focus on discovering frequent episodes from event streams, our method is general and adaptable to any class of patterns that might be of interest over the given data.

In several applications where frequent episodes have been found to be useful, share the streaming data characteristics. In neuroscience, multi electrode arrays are being used as implants to control artificial prosthetics. These interfaces interpret commands from the brain and direct external devices. Identifying controlling signals from brain is much like finding a needle in the hay stack. Large volumes of data need to be processed in real time to be able to solve this problem. Similar situations exist in telecom and computer networks where the network traffic and call logs must be analyzed to detect attacks or fraudulent activity.

A few works exist in current literature for determining frequent itemsets from a stream of transactions (e.g. see \cite{WF06}). However, they are either computationally impractical, due to worst-case assumptions, or ineffective due to strong independence assumptions. We make no statistical assumptions on the stream, independence or otherwise. We develop the error characterization of our algorithms by identifying two key properties of the data, namely, maximum rate of change and top-k separation. Our key algorithmic contribution is an adaptation of the border sets datastructures to reuse work done in previous batches when computing frequent patterns of the current batch. This reduces the candidate generation effort from $F^2$ to $FF_{new}$ (where $F$ denotes the number of frequent patterns of a particular size in the previous batch, while $F_{new}$ denotes the number of {\em newly} frequent patterns of that same size in the current batch). Experimental work demonstrates the practicality of our algorithms, both in-terms of accuracy of the returned frequent pattern sets as well as in terms of computational efficiencies.

\section{Preliminaries}
\label{sec:preliminaries}


In the framework of frequent episodes \cite{MTV97}, an {\em event sequence} is denoted as $\langle (e_1,\tau_1), \ldots, (e_n,\tau_n) \rangle$, where $(e_i,\tau_i)$ represents the $i^\mathrm{th}$ event; $e_i$ is drawn from a finite alphabet $\scrE$ of symbols (called {\em event-types}) and $\tau_i$ denotes the time-stamp of the $i^\mathrm{th}$ event, with $\tau_{i+1} \geq \tau_i$, $i=1,\ldots,(n-1)$. An {\em $\ell$-node  episode} $\alpha$ is defined by a triple $\alpha=(V_\alpha,<_\alpha,g_\alpha)$, where $V_\alpha=\{v_1,\ldots,v_\ell\}$ is a collection of $\ell$ nodes, $<_\alpha$ is a partial order over $V_\alpha$ and $g_\alpha\::\:V_\alpha \rightarrow \scrE$ is a map that assigns an event-type $g_\alpha(v)$ to each node $v\in V_\alpha$. There are two special classes of episodes: When $<_\alpha$ is total $\alpha$ is called a {\em serial episode} and when it is empty, it is called a {\em parallel episode}. An {\em occurrence} of an episode $\alpha$ is a map $h\::\:V_\alpha\rightarrow\{1, \ldots, n\}$ such that $e_{h(v)}=g(v)$ for all $v\in V_\alpha$ and for all pairs of nodes $v,v'\in V_\alpha$ such that $v<_\alpha v'$ the map $h$ ensures that $\tau_{h(v)}<\tau_{h(v')}$. Two occurrences of an episode are {\em non-overlapped} \cite{laxman06} if no event corresponding to one appears in-between the events corresponding to the other. The maximum number of non-overlapped occurrences of an episode is defined as its {\em frequency} in the event sequence. The task in frequent episode discovery is to find all patterns whose frequency exceeds a user-defined threshold. Given a frequency threshold, Apriori-style level-wise algorithms \cite{MTV97,ALVS12} can be used to obtain the frequent episodes in the event sequence. An important variant of this task is top-$k$ episode mining, where, rather than issue a frequency threshold to the mining algorithm, the user supplies the {\em number} of top frequent episodes that need to be discovered.

\begin{definition}[Top-$k$ episodes of size $\ell$]
\label{def:top-k}
The set of top-$k$ episodes of size $\ell$ is defined as the collection of all $\ell$-node episodes with frequency {\em greater than or equal to} the frequency $f^k$ of the $k^{th}$ most frequent $\ell$-node episode in the given event sequence.
\end{definition}
Note that the number of top-$k$ $\ell$-node episodes can exceed $k$, although the number of $\ell$-node episodes with frequencies strictly greater than $f^k$ is at most $(k-1)$.


\section{Problem Statement}
\label{sec:statement}


The data available (referred to as an {\em event stream}) is a potentially infinite sequence of events:
\begin{equation}
\mathcal{D} = \langle (e_1, \tau_1), (e_2, \tau_2), \ldots, (e_i, \tau_i),\ldots, (e_n, \tau_n),\ldots\rangle
\label{eq:streamD}
\end{equation}
Our goal is to find all episodes that were frequent in the recent past and to this end, we consider a sliding window model for the {\em window of interest} of the user\footnote{Streaming patterns literature has also considered other models, such as the landmark and time-fading models \cite{CKN08}, but we do not consider them in this paper.}. In this model, the user wants to determine episodes that are frequent over a window of fixed-size and terminating at the current time-tick. As new events arrive in the stream, the user's window of interest shifts, and the data mining task is to next report the frequent episodes in the  new window of interest.

\begin{figure}[t]
\centering
\includegraphics[width=\columnwidth]{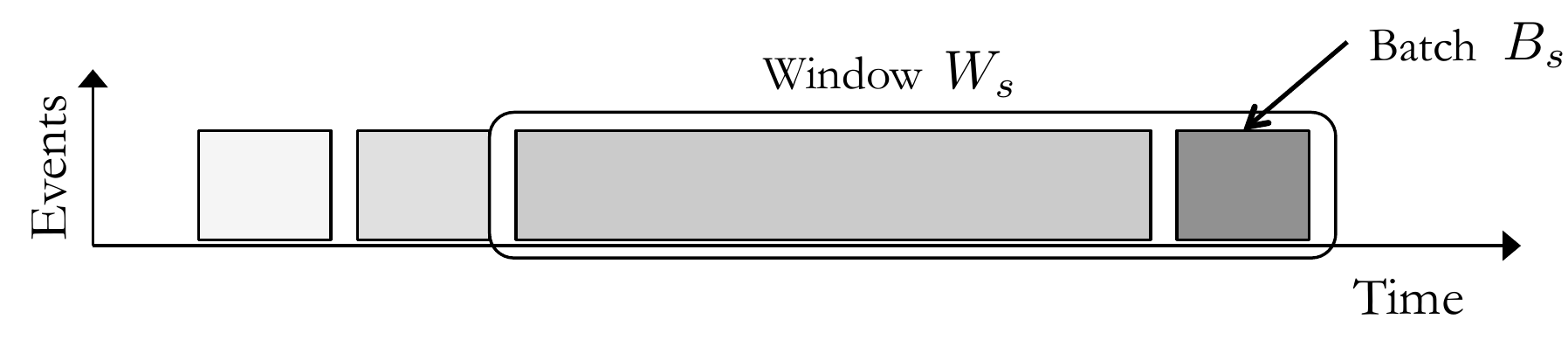}
\caption{A sliding window model for episode mining over event streams: $B_{s}$ is the most recent batch of events that arrived in the stream and $W_{s}$ is the window of interest over which the user wants to determine the set of frequent episodes.}
\label{fig:sliding-window}
\end{figure}

Typically, the window of interest is very large and cannot be stored and processed in-memory. This straightaway precludes the use of standard multi-pass algorithms for frequent episode discovery over the window of interest. Events in the stream can be organized into batches such that at any given time only the new incoming batch needs to be stored and processed in memory. This is illustrated in Fig.~\ref{fig:sliding-window}. The current window of interest is denoted by $W_{s}$ 
and the most recent batch, $B_{s}$, consists of a sequence of events in $\mathcal{D}$ with times of occurrence, $\tau_{i}$, such that,
\begin{equation}
(s-1) T_{b}\leq \tau_{i}  < s T_{b}
\end{equation}
where $T_{b}$ is the time-span of each batch and $s$ is the batch number ($s=1,2,\ldots$)\footnote{We assume that the number of events in any batch is bounded above and that we have sufficient memory to store and process all events that occur in a batch. For example, if time is integer-valued and if only one event occurs at any time-tick, then there are at most $T_{b}$ events in any batch.}. The frequency of an episode $\alpha$ in a batch $B_s$ is referred to as its {\em batch frequency} $f^s(\alpha)$. The current {\em window of interest}, $W_{s}$, consists of $m$ consecutive batches ending in batch $B_{s}$, i.e.
\begin{equation}
W_{s} = \langle B_{s-m+1},B_{s-m+2},\ldots,B_{s} \rangle
\end{equation}
\begin{definition}[Window Frequency]
The frequency of an episode $\alpha$ over window $W_s$, referred to as its {\em window frequency} and denoted by $f^{W_s}(\alpha)$, is defined as the sum of batch frequencies of $\alpha$ in $W_s$. Thus, if $f^j(\alpha)$ denotes the batch frequency of $\alpha$ in batch $B_j$, then the window frequency of $\alpha$ is given by $f^{W_s}(\alpha)=\sum_{B_j\in W_s} f^j(\alpha)$.
\label{def:window-frequency}
\end{definition}
In summary, we are given an event stream ($\scrD$), a time-span for batches ($T_b$), the number of consecutive batches that constitute the current window of interest ($m$), the desired size of frequent episodes ($\ell$) and the desired number of most frequent episodes ($k$). We are now ready to formally state the problem of discovering top-$k$ episodes in an event stream.

\begin{problem}[Streaming Top-$k$ Mining]
For each n- ew batch, $B_s$, of events in the stream, find all $\ell$-node episodes in the corresponding window of interest, $W_s$, whose window frequencies are greater than or equal to the window frequency, $f_s^k$, of $k^\mathrm{th}$ most frequent $\ell$-node episode in $W_s$.
\label{prob:streaming-topk-episodes-problem}
\end{problem}


\section{Method}

\begin{figure}[t]
\centering
\includegraphics[width=\columnwidth]{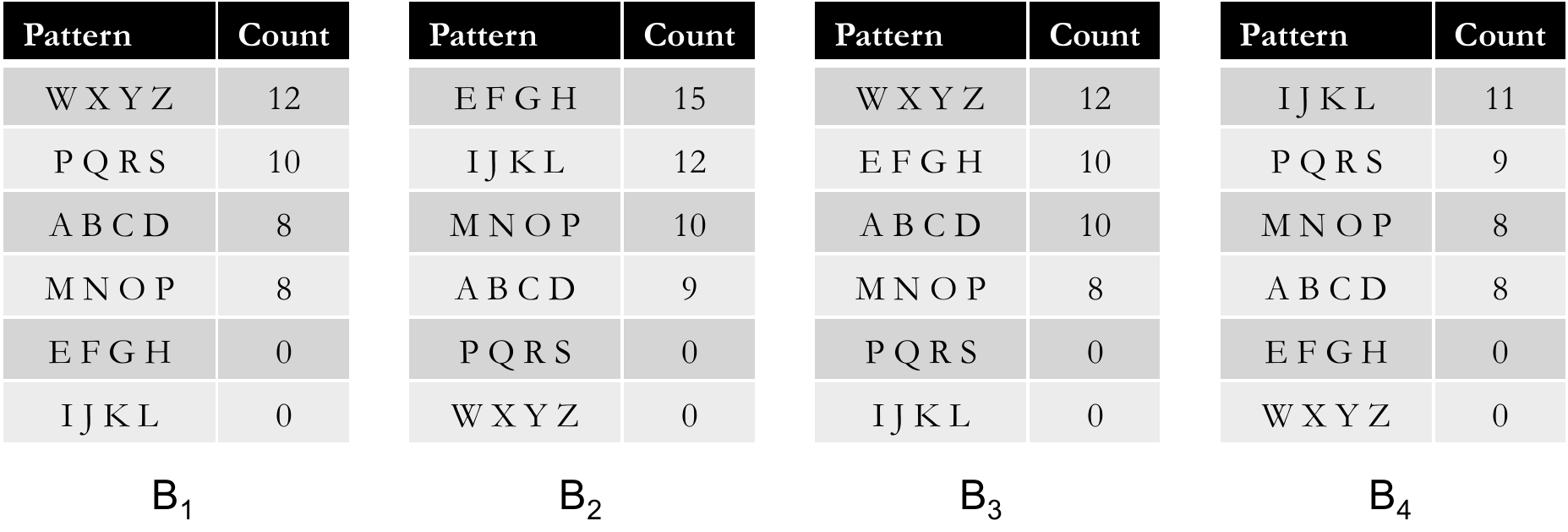}
\caption{Batch frequencies in {\em Example~\ref{eg:window-disconnect}}.}
\label{fig:top-k-connection}
\end{figure}

\begin{table}[t]
\centering
\caption{Window frequencies in {\em Example~\ref{eg:window-disconnect}}.}
\label{tab:window-counts}
\tiny
\begin{tabular}{|l|c|c|c|c|c|c|}
\hline
{\small Episode} & ABCD & MNOP & EFGH & WXYZ & IJKL & PQRS \\ \hline
{\small Window Freq} & 35 & 34 & 25 & 24 & 23 & 19\\
\hline
\end{tabular}
\end{table}

In general, the top-$k$ episodes over a window may be quite different from the top-$k$ episodes in the individual batches constituting the window. This is illustrated through an example in Fig.~\ref{fig:top-k-connection}.
\begin{example}[\label{eg:window-disconnect}Window Top-$k$ v/s Batch Top-$k$] Let $W$ be a window of four batches $B_1,\ldots, B_4$.
The episodes in each batch with corresponding batch frequencies are listed in Fig.~\ref{fig:top-k-connection}. The corresponding window frequencies (sum of each episodes' batch frequencies) are listed in Table~\ref{tab:window-counts}.
The top-2 episodes in $B_1$ are ${\rm (PQRS)}$ and ${\rm (WXYZ)}$. Similarly ${\rm (EFGH)}$ and ${\rm (IJKL)}$ are the top-2 episodes in $B_2$, and so on.
${\rm (ABCD)}$ and ${\rm (MNOP)}$ have the highest window frequencies but never appear in the top-2 of any batch -- these episodes would `fly below the radar' and go undetected if we considered only the top-2 episodes in every batch as candidates for the top-2 episodes over $W$. This example can be easily generalized to any number of batches and any $k$.
\end{example}


{\em Example~\ref{eg:window-disconnect}} highlights the main challenge in the streaming top-$k$ mining problem: while we can only store and process the most recent batch of events in the window of interest, the batchwise top-$k$ episodes may not contain sufficient informative about the top-$k$ over the entire window. It is obviously not possible to count and track all episodes (both frequent and infrequent) in every batch in the window, since the pattern space is typically very large. This brings us to the question of which episodes to select and track in every batch. How deep must we search within each batch for episodes that have potential to become top-$k$ over the window?  In this paper, we develop the formalism to answer this question. We identify two important properties of the underlying event stream which determine the design and analysis of our algorithms. These are stated in {\em Definitions~\ref{def:maximum-rate-change} \& \ref{def:topk-separation}} below.

\begin{definition} [Maximum Rate of Change, $\Delta$]
\label{def:maximum-rate-change}
The maximum change in batch frequency of any episode, $\alpha$, across any pair of consecutive batches, $B_{s}$ and $B_{s+1}$, is bounded above by $\Delta (>0)$, i.e., 
\begin{equation}
|\f{s+1} - \f{s}| \leq \Delta,
\end{equation}
and $\Delta$ is referred to as the {\em maximum rate of change}.
\end{definition}
Intuitively, $\Delta$ controls the extent of change that we may see from one batch to the next. It is trivially bounded above by the maximum number of events arriving per batch, and in practice, it is in fact much smaller.

\begin{definition}[Top-$k$ Separation of $(\varphi,\epsilon)$]
\label{def:topk-separation}
A batch $B_s$ of events is said to have a {\em top-$k$ separation of $(\varphi,\epsilon)$}, $\varphi\geq 0$, $\epsilon\geq 0$, if there are no more than $(1+\epsilon)k$ episodes with batch frequency greater than or equal to $(f_k^s - \varphi\Delta)$, where $f_k^s$ denotes the batch frequency of the $k^\mathrm{th}$ most-frequent episode in $B_s$  and $\Delta$ denotes the maximum rate of change as per {\em Definition~\ref{def:maximum-rate-change}}.
\end{definition}
This is a measure of how well-separated the frequencies of the top-$k$ episodes are relative to the rest of the episodes. We expect to see roughly $k$ episodes with batch frequencies of at least $f_s^k$ and the separation can be considered to be high (or good) if $\epsilon$ can remain small even for relatively large $\varphi$. We observe that $\epsilon$ is a non-decreasing function of $\varphi$ and that top-$k$ separation is measured relative to the maximum rate of change $\Delta$. Also, top-$k$ separation of any given batch of events is characterized through not one but several pairs of $(\varphi,\epsilon)$ since $\varphi$ and $\epsilon$ are essentially functionally related -- $\epsilon$ is typically close to zero for $\varphi=0$ and $\epsilon$ is roughly the size of the entire class of episodes (minus $k$) for $\varphi \geq f_s^k$.

We now use the maximum rate of change property to design efficient streaming algorithms for top-$k$ episode mining and show that top-$k$ separation plays a pivotal role in determining the quality of approximation that our algorithms can achieve.

\begin{lemma} 
\label{lem:fk}
Consider two consecutive batches, $B_{s}$ and $B_{s+1}$, with a maximum rate of change $\Delta$. The batch frequencies of the $k^\mathrm{th}$ most-frequent episodes in the corresponding batches are related as follows:
\begin{equation}
|\fk{s+1} - \fk{s}| \leq \Delta
\label{eq:fk-drift}
\end{equation}
\end{lemma}
\begin{proof}
There exist at least $k$ episodes in $B_{s}$ with batch frequency greater than or equal to $\fk{s}$ (by definition). Hence, there exist at least $k$ episodes  in $B_{s+1}$ with batch frequency greater than or equal to $(f^s_k - \Delta)$ (since frequency of any episode can decrease by at most $\Delta$ going from $B_{s}$ to $B_{s+1}$). Hence we must have $f^{s+1}_k \geq (f^s_k-\Delta)$. Similarly, there can be at most $(k-1)$ episodes in $B_{s+1}$ with batch frequency strictly greater than $(\fk{s}+\Delta)$. Hence we must also have $f^{s+1}_k \leq (f^s_k+\Delta)$.
\end{proof}

Next we show that if the batch frequency of an episode is known relative to $\fk{s}$ in the current batch $B_{s}$, we can bound its frequency in a later batch.

\begin{lemma} 
\label{lem:r-Delta}
Consider two batches, $B_s$ and $B_{s+r},\ r\in\mathbb{Z}$, located $r$ batches away from each other. If $\Delta$ is the maximum rate of change (as per {\em Definition~\ref{def:maximum-rate-change}}) then the batch frequency of any episode $\alpha$ in $B_{s+r}$ must satisfy the following:
\begin{enumerate}
\item If $\f{s} \geq \fk{s}$, then $\f{s+r} \geq \fk{s+r} - 2|r|\Delta$
\item If $\f{s} < \fk{s}$, then $\f{s+r} < \fk{s+r} + 2|r|\Delta$
\end{enumerate}
\end{lemma}

\begin{proof}
Since $\Delta$ is the maximum rate of change, we have $\f{s+r}\geq (\f{s}-|r|\Delta)$ and from Lemma~\ref{lem:fk}, we have $\fk{s+r} \leq (\fk{s}+|r|\Delta)$. Therefore, if $f^{s}(\alpha) \geq \fk{s}$, then
\[
\f{s+r} + |r|\Delta \geq \f{s} \geq \fk{s} \geq \fk{s+r}-|r|\Delta
\]
which implies $\f{s+r} \geq \fk{s+r} - 2|r|\Delta$.
Similarly, if  $\f{s} < \fk{s}$, then
\[
\f{s+r} - |r|\Delta \leq \f{s} < \fk{s} \leq \fk{s+r}+|r|\Delta
\]
which implies $\f{s+r} < \fk{s+r} + 2|r|\Delta$.
\end{proof}

{\em Lemma~\ref{lem:r-Delta}} gives us a way to track episodes that have potential to be in the top-$k$ of future batches. This is an important property which our algorithm exploits and we recorded this as a remark below.
\begin{remark}
The top-$k$ episodes of batch, $B_{s+r},\ r\in\mathbb{Z}$, must have batch frequencies of at least $(\fk{s} - 2|r|\Delta)$ in batch $B_{s}$. Specifically, the top-$k$ episodes of $B_{s+1}$ must have batch frequencies of at least $(\fk{s}-2\Delta)$ in $B_{s}$.
\label{rem:batch-topk}
\end{remark}

%

Based on the maximum rate of change property we can derive a necessary condition for any episode to be top-$k$ over a window. The following theorem prescribes the minimum batch frequencies that an episode must satisfy if it is a top-$k$ episode over the window $W_s$.
\begin{theorem}[Exact Top-$k$ over $W_s$]
\label{thm:topk-mine}
An episode, $\alpha$, can be a top-$k$ episode over window $W_{s}$ only if its batch frequencies satisfy  $f^{s'}(\alpha) \geq (\fk{s'}-2(m-1)\Delta)$ $\forall B_{s'} \in W_{s}$.
\end{theorem}
\begin{proof} 
Consider an episode $\beta$ for which $\fb{s'} < (\fk{s'}-2(m-1)\Delta)$ in batch $B_{s'}\in W_s$. Let $\alpha$ be any top-$k$ episode of $B_{s'}$.
In any other batch $B_p\in W_s$, we have 
\begin{align}
\label{eq:alph1}
\f{p} &\geq \f{s'}-|p-s'|\Delta \nonumber\\
	  &\geq \fk{s'}-|p-s'|\Delta
\end{align}
and 
\begin{align}
\label{eq:alph2}
\fb{p} &\leq \fb{s'}+|p-s'|\Delta\nonumber\\ 
       &< (\fk{s'}-2(m-1)\Delta)+|p-s'|\Delta
\end{align}
Applying $|p-s'|\leq (m-1)$ to the above, we get
\begin{align}
\f{p} \geq \fk{s'}-(m-1)\Delta > \fb{p}
\label{eq:alpha-greater-than-beta}
\end{align}
This implies $\fb{W_{s}} < \f{W_{s}}$ for every top-$k$ episode $\alpha$ of $B_{s'}$.
Since there are at least $k$ top-$k$ episodes in $B_{s'}$, $\beta$ cannot be a top-$k$ episode over the window $W_{s}$.
\end{proof}


Based on {\em Theorem~\ref{thm:topk-mine}} we can have the following simple algorithm for obtaining the top-$k$ episodes over a window: Use a traditional level-wise approach to find all episodes with a batch frequency of at least $(f_1^k-2(m-1)\Delta)$ in the first batch ($B_1$), simply accumulate their corresponding batch frequencies over all $m$ batches of $W_s$ and report the episodes with the $k$ highest window frequencies over $W_s$. This approach is guaranteed to give us the exact top-$k$ episodes over $W_s$. Further, in order to report the top-$k$ over the next sliding window $W_{s+1}$, we need to consider all episodes with batch frequency of at least $(f_2^k-2(m-1)\Delta)$ in the second batch and track them over all batches of $W_{s+1}$, and so on. Thus, an exact solution to {\em Problem~\ref{prob:streaming-topk-episodes-problem}} would require running a level-wise episode mining algorithm in every batch, $B_s$, $s=1,2,\ldots$, with a frequency threshold of $(f_s^k-2(m-1)\Delta)$.


\subsection{Class of $(v,k)$-Persistent Episodes}
\label{sec:persistance}

{\em Theorem~\ref{thm:topk-mine}} characterizes the minimum batchwise computation needed in order to obtain the exact top-$k$ episodes over a sliding window. This is effective when $\Delta$ and $m$ are small (compared to $f_s^k$). However, the batchwise frequency thresholds can become very low in other settings, making the processing time per-batch as well as the number of episodes to track over the window to become impractically high. To address this issue, we introduce a new class of episodes called {\em $(v,k)$-persistent episodes} which can be computed efficiently by employing higher batchwise thresholds. Further, we show that these episodes can be used to approximate the true top-$k$ episodes over the window and the quality of approximation is characterized in terms of the top-$k$ separation property (cf.~{\em Definition~\ref{def:topk-separation}}).

\begin{definition}[$(v,k)$-Persistent Episode] A pattern is said to be {\em $(v,k)$-persistent} over window $W_{s}$ if it is a top-$k$ episode in {\em at least} $v$ batches of $W_{s}$.
\label{def:persitent-episode}
\end{definition}


\begin{problem}[Mining $(v,k)$-Persistent Episodes]
For each new batch, $B_s$, of events in the stream, find all $\ell$-node $(v,k)$-persistent episodes in the corresponding window of interest, $W_s$.
\label{prob:vk-persistent-episodes-problem}
\end{problem}

\begin{theorem}
An episode, $\alpha$, can be $(v,k)$-persistent over the window $W_{s}$ only if its batch frequencies satisfy  $f^{s'}(\alpha) \geq (\fk{s'}-2(m-v)\Delta)$ for every batch $B_{s'} \in W_{s}$.
\label{thm:persistent-episodes}
\end{theorem}
\begin{proof}
Let $\alpha$ be $(v,k)$-persistent over $W_s$ and let $V_\alpha$ denote the set of batches in $W_s$ in which $\alpha$ is in the top-$k$. For any $B_q\notin V_\alpha$ there exists $B_{\widehat{p}(q)}\in V_\alpha$ that is {\em nearest to} $B_q$. Since $|V_\alpha|\geq v$, we must have $|\widehat{p}(q)-q| \leq (m-v)$. Applying {\em Lemma~\ref{lem:r-Delta}} we then get
$f^q(\alpha)\geq f^q_k-2(m-v)\Delta$ for all $B_q\notin V_\alpha$.
\end{proof}
{\em Theorem~\ref{thm:persistent-episodes}} gives us the necessary conditions for computing all $(v,k)$-persistent episodes over sliding windows in the stream. The batchwise threshold required for $(v,k)$-persistent episodes depends on the parameter $v$. For $v=1$, the threshold coincides with the threshold for exact top-$k$ in {\em Theorem~\ref{thm:topk-mine}}. The threshold increases linearly with $v$ and is highest at $v=m$ (when the batchwise threshold is same as the corresponding batchwise top-$k$ frequency).

The algorithm for discovering $(v,k)$-persistent episodes follows the same general lines as the one described earlier for exact top-$k$ mining, only that we now apply higher batchwise thresholds: For each new batch, $B_s$, entering the stream, use a standard level-wise episode mining algorithm to find all episodes with batch frequency of at least $(f_s^k-2(m-v)\Delta)$. (We provide more details of our algorithm later in Sec.~\ref{sec:incremental-algorithm}). First, we investigate the quality of approximation of top-$k$ that $(v,k)$-persistent episodes offer and show that the number of errors is closely related to the degree of top-$k$ separation in the data.

\subsubsection{Top-$k$ Approximation}
\label{sec:topk-approximation}

The main idea here is that, under a maximum rate of change $\Delta$ and a top-$k$ separation of $(\varphi,\epsilon)$, there cannot be too many distinct episodes which are not $(v,k)$-persistent, while having sufficiently high window frequencies. To this end, we first compute a lower-bound  ($f_L$) on the  window frequencies of $(v,k)$-persistent episodes and an upper-bound ($f_U$) on the window frequencies of episodes that are {\em not} $(v,k)$-persistent (cf.~{\em Lemmas~\ref{lem:fL} \& \ref{lem:fU}}).

\begin{lemma}
If episode $\alpha$ is $(v,k)$-persistent over a window, $W_s$, then its window frequency, $f^{W_s}(\alpha)$, must satisfy the following lower-bound:
\begin{equation}
f^{W_s}(\alpha) \geq \sum_{B_{s'}} f^{s'}_k - (m-v)(m-v+1)\Delta  \stackrel{\rm def}{=} f_L
\end{equation}
\label{lem:fL}
\end{lemma}
\begin{proof}
Consider episode $\alpha$ that is $(v,k)$-persistent over $W_s$ and let $V_\alpha$ denote the batches of $W_s$ in which $\alpha$ is in the top-$k$. The window frequency of $\alpha$ can be written as
\begin{eqnarray}
f^{W_s}(\alpha) &=& \sum_{B_p\in V_\alpha} f^p(\alpha) + \sum_{B_q\in W_s\setminus V_\alpha} f^q(\alpha) \nonumber\\
&\geq& \sum_{B_p\in V_\alpha} f^p_k + \sum_{B_q\in W_s\setminus V_\alpha} f^q_k -2|\widehat{p}(q)-q|\Delta \nonumber\\
&=& \sum_{B_{s'}\in W_s} f^{s'}_k - \sum_{B_q\in W_s\setminus V_\alpha} 2|\widehat{p}(q)-q|\Delta \label{eq:lemfL1}
\end{eqnarray}
where $B_{\widehat{p}(q)}\in V_\alpha$ denotes the batch nearest $B_q$ where $\alpha$ is in the top-$k$. Since $|W_s\setminus V_\alpha|\leq (m-v)$, we must have
\begin{eqnarray}
\sum_{B_q\in W_s\setminus V_\alpha} |\widehat{p}(q)-q| &\leq& (1+2+\cdots+(m-v)) \nonumber\\
&=&\frac{1}{2}(m-v)(m-v+1) \label{eq:lemfL2}
\end{eqnarray}
Putting together (\ref{eq:lemfL1}) and (\ref{eq:lemfL2}) gives us the lemma.
\end{proof}

\begin{lemma}
If episode $\beta$ is not $(v,k)$-persistent over a window, $W_s$, then its window frequency, $f^{W_s}(\beta)$, must satisfy the following upper-bound:
\begin{equation}
f^{W_s}(\beta) < \sum_{B_{s'}} f^{s'}_k + v (v+1)\Delta  \stackrel{\rm def}{=} f_U
\end{equation}
\label{lem:fU}
\end{lemma}
\begin{proof}
Consider episode $\beta$ that is not $(v,k)$-persistent over $W_s$ and let $V_\beta$ denote the batches of $W_s$ in which $\beta$ is in the top-$k$. The window frequency of $\beta$ can be written as:
\begin{eqnarray}
f^{W_s}(\beta) &=& \sum_{B_p\in V_\beta} f^p(\beta) + \sum_{B_q\in W_s\setminus V_\beta} f^q(\beta) \nonumber\\
&<& \sum_{B_p\in V_\beta} f^p_k + 2|\widehat{p}(q)-q|\Delta + \sum_{B_q\in W_s\setminus V_\beta} f^q_k \nonumber\\
&=& \sum_{B_{s'}\in W_s} f^{s'}_k + \sum_{B_p\in V_\beta} 2|\widehat{q}(p)-p|\Delta \label{eq:lemfU1}
\end{eqnarray}
where $B_{\widehat{q}(p)}\in W_s\setminus V_\beta$ denotes the batch nearest $B_p$ where $\beta$ is not in the top-$k$. Since $|V_\beta| < v$, we must have
\begin{eqnarray}
\sum_{B_p\in V_\beta} |\widehat{q}(p)-p| &\leq& (1+2+\cdots+(v-1)) \nonumber\\
&=&\frac{1}{2}v(v+1) \label{eq:lemfU2}
\end{eqnarray}
Putting together (\ref{eq:lemfU1}) and (\ref{eq:lemfU2}) gives us the lemma.
\end{proof}

It turns out that $f_U > f_L\ \forall v,$ $1\leq v \leq m$, and hence there is always a possibility for some episodes which are not $(v,k)$-persistent to end up with higher window frequencies than one or more $(v,k)$-persistent episodes. We observed a specific instance of this kind of `mixing' in our motivating example as well (cf.~{\em Example~\ref{eg:window-disconnect}}). This brings us to the top-$k$ separation property that we introduced in {\em Definition~\ref{def:topk-separation}}. Intuitively, if there is sufficient separation of the top-$k$ episodes from the rest of the episodes in every batch, then we would expect to see very little mixing. As we shall see, this separation need not occur exactly at $k^\mathrm{th}$ most-frequent episode in every batch, somewhere close to it is sufficient to achieve a good top-$k$ approximation.

\begin{definition}[Band Gap Episodes, $\scrG_\varphi$]
In any batch $B_{s'}\in W_s$, the half-open frequency interval $[f_{s'}^k - \varphi \Delta,\ f_{s'}^k)$ is called the {\em band gap} of $B_{s'}$. The corresponding set, $\scrG_\varphi$, of {\em band gap episodes} over the window $W_s$, is defined as the collection of all episodes with batch frequencies in the band gap of at least one $B_{s'}\in W_s$.
\label{def:band-gap-patterns}
\end{definition}

The main feature of $\scrG_\varphi$ is that, if $\varphi$ is large-enough, then the only episodes which are not $(v,k)$-persistent but that can still mix with $(v,k)$-persistent episodes are those belonging to $\scrG_\varphi$. This is stated formally in the next lemma.
\begin{lemma}
If $\frac{\varphi}{2} > \max \{1, (1-\frac{v}{m})(m-v+1)\}$, then any episode $\beta$ that is {\em not} $(v,k)$-persistent over $W_s$, can have $f^{W_s}(\beta) \geq f_L$ only if $\beta\in\scrG_\varphi$.
\label{lem:freq-G}
\end{lemma}
\begin{proof}
If an episode $\beta$ is {\em not} $(v,k)$-persistent over $W_s$ then there exists a batch $B_{s'}\in W_s$ where $\beta$ is not in the top-$k$. Further, if $\beta\notin\scrG_\varphi$ then we must have $f_{s'}(\beta)<f^k_{s'}-\varphi\Delta$. Since $\varphi>2$, $\beta$ cannot be in the top-$k$ of any neighboring batch of $B_{s'}$, and hence, it will stay below $f_{s'}^k-\varphi\Delta$ for all $B_{s'}\in W_s$, i.e.,
\begin{equation*}
f^{W_s}(\beta) < \sum_{B_{s'}\in W_s} f^k_{s'} - m\varphi\Delta.
\end{equation*}
The Lemma follows from the given condition $\frac{\varphi}{2} > (1-\frac{v}{m})(m-v+1)$.
\end{proof}
The number of episodes in $\scrG_\varphi$ is controlled by the top-$k$ separation property, and since many of the non-persistent episodes which can mix with persistent ones must spend not one, but several batches in the band gap, the number of unique episodes that can cause such errors is bounded. {\em Theorem~\ref{thm:topk-approximation}} is our main result about quality of top-$k$ approximation that $(v,k)$-persistence can achieve.
\begin{theorem}[Quality of Top-$k$ Approximation]
Let every batch $B_{s'}\in W_s$ have a top-$k$ separation of $(\varphi,\epsilon)$ with $\frac{\varphi}{2} > \max \{1, (1-\frac{v}{m})(m-v+1)\}$. Let $\scrP$ denote the set of all $(v,k)$-persistent episodes over $W_s$. If $|\scrP| \geq k$, then the top-$k$ episodes over $W_s$ can be determined from $\scrP$ with an error of no more than $\left(\frac{\epsilon k m}{\mu}\right)$ episodes, where $\mu = \min \{ m-v+1, \frac{\varphi}{2}, \frac{1}{2}(\sqrt{1+2m\varphi} - 1)\}$.
\label{thm:topk-approximation}
\end{theorem}
\begin{proof}
By top-$k$ separation, we have a maximum of $(1+\epsilon)k$ episodes in any batch $B_{s'}\in W_s$, with batch frequencies greater than or equal to $f_{s'}^k-\varphi\Delta$. Since at least $k$ of these must belong to the top-$k$ of the $B_{s'}$, there are no more than $\epsilon k$ episodes that can belong to the band gap of $B_{s'}$. Thus, there can be no more than a total of $\epsilon km$ episodes over all $m$ batches of $W_s$ that can belong to $\scrG_\varphi$.

Consider any $\beta\notin\scrP$ with $f^{W_s}(\beta)\geq f_L$ -- these are the only episodes whose window frequencies can exceed that of any $\alpha\in\scrP$ (since $f_L$ is the minimum window frequency of any $\alpha$). If $\mu$ denotes the minimum number of batches in which $\beta$ belongs to the band gap, then there can be at most $\left(\frac{\epsilon k m}{\mu}\right)$ such {\em distinct} $\beta$. Thus, if $|\scrP|\geq k$, we can determine the set of top-$k$ episodes over $W_s$ with error no more than $\left(\frac{\epsilon k m}{\mu}\right)$ episodes.

There are now two cases to consider to determine $\mu$: (i)~$\beta$ is in the top-$k$ of some batch, and (ii)~$\beta$ is not in the top-$k$ of any batch.

Case (i): Let $\beta$ be in the top-$k$ of $B_{s'}\in W_s$. Let $B_{s''}\in W_s$ be $t$ batches away from $B_{s'}$. Using {\em Lemma~\ref{lem:r-Delta}} we get $f^{s''}(\beta) \geq f_{s''}^k - 2t\Delta$. The minimum $t$ for which $(f_{s''}^k - 2t\Delta < f_s^k - \varphi\Delta)$ is $\left(\frac{\varphi}{2}\right)$. Since $\beta\notin\scrP$, $\beta$ is below the top-$k$ in at least $(m-v+1)$ batches. Hence $\beta$ stays in the band gap of at least $\min\{m-v+1,\frac{\varphi}{2}\}$ batches of $W_s$.

Case (ii): Let $V_G$ denote the set of batches in $W_s$ where $\beta$ lies in the band gap and let $|V_G|=g$. Since $\beta$ does not belong to top-$k$ of any batch, it must stay below the band gap in all the $(m-g)$ batches of $(W_s\setminus V_G)$. Since $\Delta$ is the maximum rate of change, the window frequency of $\beta$ can be written as follows:
\begin{eqnarray}
f^{W_s}(\beta) &=& \sum_{B_p\in V_G} f^p(\beta) + \sum_{B_q\in W_s\setminus V_G} f^q(\beta) \nonumber\\
&<& \sum_{B_p\in V_G} f^p(\beta) + \sum_{B_q\in W_s\setminus V_G} (f^k_q-\varphi\Delta) \label{eq:thm-topk-pf1}
\end{eqnarray}
Let $B_{\widehat{q}(p)}$ denote the batch in $W_s\setminus V_G$ that is nearest to $B_p\in V_G$. Then we have:
\begin{eqnarray}
f^p(\beta) &\leq& f^{\widehat{q}(p)}(\beta) + |p-\widehat{q}(p)|\Delta \nonumber\\
&<& f^k_{\widehat{q}(p)} -\varphi\Delta + |p-\widehat{q}(p)|\Delta \nonumber\\
&<& f^k_p -\varphi\Delta + 2 |p-\widehat{q}(p)|\Delta \label{eq:thm-topk-pf2}
\end{eqnarray}
where the second inequality holds because $\beta$ is below the band gap in $B_{\widehat{q}(p)}$ and (\ref{eq:thm-topk-pf2}) follows from {\em Lemma~\ref{lem:fk}}. Using (\ref{eq:thm-topk-pf2}) in (\ref{eq:thm-topk-pf1}) we get
\begin{eqnarray}
f^{W_s}(\beta) &<& \sum_{B_{s'}\in W_s} f^k_{s'} -m\varphi\Delta+ \sum_{B_p\in V_G} 2|p-\widehat{q}(p)|\Delta \nonumber\\
&<& \sum_{B_{s'}\in W_s} f^k_{s'} -m\varphi\Delta+ 2(1+2+\cdots+g)\Delta\nonumber\\
&=& \sum_{B_{s'}\in W_s} f^k_{s'} -m\varphi\Delta+ g(g+1)\Delta = \mathrm{UB}\label{eq:thm-topk-pf3}
\end{eqnarray}
The smallest $g$ for which $(f^{W_s}(\beta)\geq f_L)$ is feasible can be obtained by setting $\mathrm{UB}\geq f_L$. Since $\frac{\varphi}{2} > (1-\frac{v}{m})(m-v+1)$, $\mathrm{UB} \geq f_L$ implies
\begin{eqnarray}
\sum_{B_{s'}\in W_s} f^k_{s'} - m\varphi\Delta + g(g+1)\Delta
&>& \sum_{B_{s'}\in W_s} f^k_{s'} - \frac{m\varphi\Delta}{2} \nonumber
\end{eqnarray}
Solving for $g$, we get $g\geq \frac{1}{2}(\sqrt{1+2m\varphi}-1)$. Combining cases (i) and (ii), we get $\mu=\min \{ m-v+1, \frac{\varphi}{2}, \frac{1}{2}(\sqrt{1+2m\varphi} - 1)\}$.
\end{proof}

{\em Theorem~\ref{thm:topk-approximation}} shows the relationship between the extent of top-$k$ separation required and quality of top-$k$ approximation that can be obtained through $(v,k)$-persistent episodes. In general, $\mu$ increases with $\frac{\varphi}{2}$ until the latter starts to dominate the other two factors, namely, $(m-v+1)$ and $\frac{1}{2}(\sqrt{1+2m\varphi}-1)$. The theorem also brings out the tension between the persistence parameter $v$ and the quality of approximation. At smaller values of $v$, the algorithm mines `deeper' within each batch and so we expect fewer errors with respect to the true top-$k$ epispodes. On the other hand, deeper mining within batches is computationally more intensive, with the required effort approaching that of exact top-$k$ mining as $v$ approaches 1. Finally, we use {\em Theorem~\ref{thm:topk-approximation}} to derive error-bounds for three special cases; first for $v=1$, when the batchwise threshold is same as that for exact top-$k$ mining as per {\em Theorem~\ref{thm:topk-mine}}; second for $v=m$, when the batchwise threshold is simply the batch frequency of the $k^{\rm th}$ most-frequent episode in the batch; and third, for $v=\left\lfloor\frac{m+1}{2}\right\rfloor$, when the batchwise threshold lies midway between the thresholds of the first two cases.
\begin{corollary}

Let every batch $B_{s'}\in W_s$ have a top-$k$ separation of $(\varphi,\epsilon)$ and let $W_s$ contain at least $m\geq 2$ batches. Let $\scrP$ denote the set of all $(v,k)$-persistent episodes over $W_s$. If we have $|\scrP| \geq k$, then the maximum number of errors in the top-$k$ episodes derived from $\scrP$, for three different choices of $v$, is given by:
\begin{enumerate}
\item $\left(\frac{\epsilon k m}{m-1}\right)$, for $v=1$, if $\frac{\varphi}{2} > (m-1)$
\item $(\epsilon k m)$, for $v=m$, if $\frac{\varphi}{2} > 1$
\item $\left(\frac{4 \epsilon k m^2}{m^2-1}\right)$, for $v=\left\lfloor\frac{m+1}{2}\right\rfloor$, if $\frac{\varphi}{2} > \frac{1}{m}\left\lceil \frac{m-1}{2}\right\rceil\left\lceil\frac{m+1}{2} \right\rceil$
\end{enumerate}
\label{cor:topk-approximation-heuristic-v}
\end{corollary}
\begin{proof}
We show the proof only for $v=\left\lfloor\frac{m+1}{2}\right\rfloor$. The cases of $v=1$ and $v=m$ are obtained immediately upon application of {\em Theorem~\ref{thm:topk-approximation}}.

Fixing $v=\left\lfloor\frac{m+1}{2}\right\rfloor$ implies $(m-v)=\left\lceil\frac{m-1}{2}\right\rceil$. For $m\geq 2$, $\frac{\varphi}{2} > \frac{1}{m}\left\lceil \frac{m-1}{2}\right\rceil\left\lceil\frac{m+1}{2} \right\rceil$ implies $\frac{\varphi}{2} > \max \{1, (1-\frac{v}{m})(m-v+1)\}$. Let $t_{\rm min} = \min\{m-v+1,\frac{\varphi}{2}\}$. The minimum value of $t_{\rm min}$ is governed by
\begin{eqnarray}
t_{\rm min} &\geq& \min \left\{ \left\lceil\frac{m+1}{2}\right\rceil, \frac{1}{m}\left\lceil \frac{m-1}{2}\right\rceil \left\lceil\frac{m+1}{2} \right\rceil \right\}\nonumber\\
&=& \frac{1}{m}\left\lceil \frac{m-1}{2}\right\rceil \left\lceil\frac{m+1}{2} \right\rceil \nonumber\\
&\geq& \left(\frac{m^2-1}{4m}\right)
\label{eq:thm-topk-heuristic-pf1}
\end{eqnarray}
Let $g_{\rm min}=\frac{1}{2}(\sqrt{1+2m\varphi} - 1)$. $\varphi > \frac{2}{m}\left\lceil \frac{m-1}{2}\right\rceil\left\lceil\frac{m+1}{2} \right\rceil$ implies $g_{\rm min}>\left(\frac{m-1}{2}\right)$. From {\em Theorem~\ref{thm:topk-approximation}} we have
\begin{equation*}
\mu=\min\{t_{\rm min}, g_{\rm min}\} \geq \left(\frac{m^2-1}{4m}\right)
\end{equation*}
and hence the number of errors is no more than $\left(\frac{4\epsilon k m^2}{m^2-1}\right)$.
\end{proof}

\subsection{Incremental Algorithm}
\label{sec:incremental-algorithm}

In this section we present an efficient algorithm for incrementally mining patterns with frequency $\geq (\fk{s} - \theta) $. From our formalism, the value of $\theta$ is specified by the type of patterns we want to mine. For $(v,k)$ persistence, $\theta = 2(m-v)\Delta$ whereas for mining the exact top-$k$ the threshold is $2(m-1)\Delta$.

\begin{figure}[htbp]
\centering
\includegraphics[width=\columnwidth]{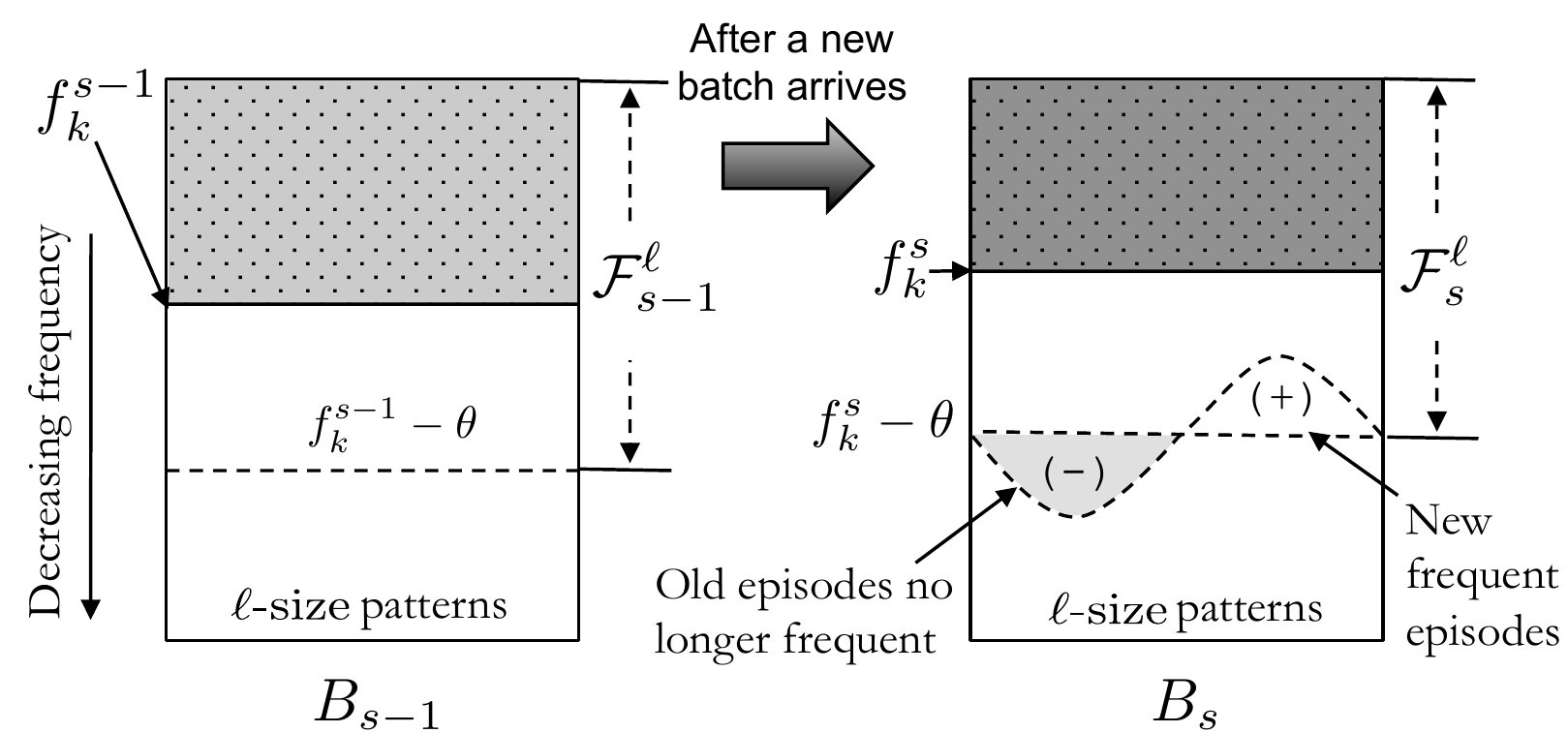}
\caption{The set of frequent patterns can be incrementally updated as new batches arrive.}
\label{fig:batch-update}
\end{figure}

Recall that the goal of our mining task is to report frequent patterns of size $\ell$. 
After processing the data in the batch $B_{s-1}$, we desire all patterns with frequency greater than $(\fk{s-1}-\theta)$. Algorithmically this is achieved by first setting a high frequency threshold and mining for patterns using  the classical level wise Apriori method~\cite{srikant-agrawal}. 
If the number of patterns of size-$\ell$ is less than $k$, the support threshold is decreased and the mining repeated until atleast $k$ $\ell$-size patterns are found. At this point $\fk{s}$ is known. The mining process is repeated once more with the frequency threshold $(\fk{s}-\theta)$. Doing this entire procedure for every new batch can be expensive and wasteful. After seeing the first batch of the data, whenever a new batch arrives we have information about the patterns that were frequent in the previous batch. This can be exploited to incrementally and efficiently update the set of frequent episodes in the new batch. The intuition behind this is that the frequencies of the majority of episodes do not change much from one batch to the next. As a result a small number of episode fall below the new support threshold in the new batch. There is also the possibility of some new episodes becoming frequent. This is illustrated in Figure~\ref{fig:batch-update}. In order to efficiently find these sets of episodes, we need to maintain additional information that allows us to avoid full-blown candidate generation.
We show that this state information is a by-product of Apriori algorithm and therefore any extra processing is unnecessary.

In the Apriori algorithm, frequent patterns are discovered iteratively, in ascending order of their size and it is often referred to as a levelwise procedure. The procedure alternates between counting and candidate generation. First a set $C^{i}$ of candidate $i$-size patterns is created by joining the frequent $(i-1)$-size itemsets found in the previous iteration. Then the data is scanned for determining the frequency or count of each candidate pattern and the frequent $i$-size patterns are extracted from the candidates. An interesting observation is that all candidate episodes that are not frequent constitute the negative border of the frequent lattice. This is true because, in the Apriori algorithm, a candidate pattern is generated only when all its subpatterns are frequent. The usual approach is to discard the border. 
For our purposes, the patterns in the border contain the information required to identify the change in the frequent sets from one batch to the next.

The pseudocode for incrementally mining frequent patterns in batches is listed in Algorithm~\ref{alg:mine-top-k}. 
Let the frequent episodes of size-$i$ be denoted by $\mathcal{F}_{s}^{i}$. Similarly, the border episodes of size-$i$ are denoted by $\mathcal{B}_{s}^{i}$. The frequency threshold used in each batch is $\fk{s}-\theta$. In the first batch of data, the top-$k$ patterns are found by progressively lowering the frequency threshold $f_{min}$ by a small amount $\epsilon$ (Lines 1-8). Once atleast $k$ patterns of size $\ell$ are found, $f_{k}^{s}$ is determined and the mining procedure repeated with a threshold of $\fk{s}-\theta$. The border patterns generated during level wise mining are retained.

For subsequent batches, first $\fk{s}$ is determined. As shown in Remark~\ref{rem:batch-topk}, if $\theta \geq 2\Delta$, then the set of frequent patterns $\mathcal{F}_{s-1}^{\ell}$ in batch $B_{s-1}$ contains all patterns that can be frequent in the next batch $B_{s}$. Therefore simply updating the counts of all patterns in $\mathcal{F}_{s-1}^{\ell}$ in the batch $B_{s}$ and picking the $k^{th}$ highest frequency gives $\fk{s}$ (Lines 10-11). The new frequency threshold $f_{min}$ is set to be $\fk{s}-\theta$.
The procedure, starting from bottom (size-1 patterns) updates the lattice for $B_{s}$. The data is scanned to determine the frequency of new candidates together with the frequent and border patterns from the lattice (Line 15-18). In the first level (patterns of size 1), the candidate set is empty. After counting, the patterns from the frequent set $\mathcal{F}_{s-1}^{\ell}$ that continue to be frequent in the new batch are added to $\mathcal{F}_{s}^{\ell}$. But if a pattern is no longer frequent it is marked as a border set and all its super episodes are deleted (Lines 19-24). This ensures that only border patterns are retained in the lattice.
All patterns, either from the border set or the new candidate set, that are found to be frequent are added to $\mathcal{F}_{s}^{\ell}$. Such episodes are also added to $F_{new}^{i}$. Any remaining infrequent patterns belong to border set because otherwise they would have atleast one of infrequent subpatterns and would have been deleted at a previous level (Line 24). These patterns are added to $\mathcal{B}_{s}^{\ell}$ (Line 30).
The candidate generation step is required to fill out the missing parts of the frequent lattice. We want to avoid a full blown candidate generation. Note that if a pattern is frequent in $B_{s-1}$ and $B_{s}$ then all its subpatterns are also frequent in both $B_{s}$ and $B_{s-1}$. Any new pattern ($\not\in \mathcal{F}_{s-1}^{\ell} \cup \mathcal{B}_{s-1}^{\ell}$) that turns frequent in $B_{s}$, therefore, must have atleast one subpattern that was not frequent in $B_{s-1}$ but is frequent in $B_{s}$. All such patterns are listed in $F_{new}^{i}$.
The candidate generation step (Line 31)  for the next level generates only candidate patterns with atleast one subpattern $\in F_{new}^{i}$.  This greatly restricts the number of candidates generated at each level without compromising the completeness of the results.

The space and time complexity of the candidate generation is now $O(|F_{new}^{i}|.|\mathcal{F}_{s}^{i}|)$ instead of $O(|\mathcal{F}_{s}^{i}|^{2})$ and in most practical cases $|F_{new}^{i}| \ll |\mathcal{F}_{s}^{i}|$. This is crucial in a streaming application where processing rate must match the data arrival rate.

\begin{algorithm}[!ht]
\caption{Mine top-$k$ $v$-persistent patterns.}
\label{alg:mine-top-k}
\begin{algorithmic}[1]
\small
\REQUIRE{A new batch of events $B_{s}$, the lattice of frequent and border patterns $(\mathcal{F}^{*}_{s-1}, \mathcal{B}^{*}_{s-1})$, and parameters $k$ and $\theta$}
\ENSURE{The lattice of frequent and border patterns $(\mathcal{F}^{*}_{s}, \mathcal{B}^{*}_{s})$}
\IF {$s = 1$}
	\STATE $f_{min} = $ high value
	\WHILE{$|\mathcal{F}_{s}^{\ell}| < k$}
		\STATE Mine patterns with frequency $\geq f_{min}$
		\STATE $f_{min} = f_{min} - \epsilon$
	\ENDWHILE
	\STATE $\fk{s} = $ frequency of the $k^{th}$ most frequent pattern $\in \mathcal{F}_{s}^{\ell}$
	\STATE Mine $\ell$-size patterns in $B_{s}$ with frequency threshold $f_{min} = \fk{s} - \theta$
	\STATE Store the frequent and border patterns (of size = $1\ldots \ell$) in $(\mathcal{F}^{*}_{s}, \mathcal{B}^{*}_{s})$
\ELSE
	\STATE {\bf CountPatterns}$(\mathcal{F}^{\ell}_{s-1}, B_{s})$
	\STATE Set $\fk{s} =$ frequency $k^{th}$ highest frequency (pattern $\in \mathcal{F}^{\ell}_{s-1}$)
	\STATE Set frequency threshold for $B_{s}$, $f_{min} =(\fk{s}-\theta)$
	\STATE $\mathcal{C}^{1} = \varphi$ \COMMENT{New candidate patterns of size $=1$}
	\FOR{$i = 1\ldots\ell-1$}
		\STATE $\mathcal{F}^{i}_{s} = \varphi$ \COMMENT{Frequent patterns of size $i$}
		\STATE $\mathcal{B}^{i}_{s} = \varphi$ \COMMENT{Border patterns of size $i$}
		\STATE $F^{i}_{new} = \varphi$ \COMMENT{List of newly frequent Patterns}
		\STATE {\bf CountPatterns}$(\mathcal{F}^{i}_{s-1} \cup \mathcal{B}^{i}_{s-1} \cup \mathcal{C}^{i}, B_{s})$
		\FOR{$\alpha \in \mathcal{F}^{i}_{s-1}$}
			\IF{$\f{s} \geq f_{min}$}
				\STATE $\mathcal{F}^{i}_{s} = \mathcal{F}^{i}_{s} \cup \{\alpha \}$
			\ELSE
				\STATE $\mathcal{B}^{i}_{s} = \mathcal{B}^{i}_{s} \cup \{\alpha \}$
				\STATE Delete all its super-patterns from $(\mathcal{F}^{*}_{s-1}, \mathcal{B}^{*}_{s-1})$
 			\ENDIF
		\ENDFOR
		\FOR{$\alpha \in \mathcal{B}^{i}_{s-1} \cup \mathcal{C}^{i}$}
			\IF{$\f{s} \geq f_{min}$}
				\STATE $\mathcal{F}^{i}_{s} = \mathcal{F}^{i}_{s} \cup \{\alpha \}$
				\STATE $F^{i}_{new} = F^{i}_{new} \cup \{\alpha\}$
			\ELSE
				\STATE $\mathcal{B}^{i}_{s} = \mathcal{B}^{i}_{s} \cup \{\alpha \}$
 			\ENDIF
		\ENDFOR
		\STATE $C^{i+1} = \mbox{\bf GenerateCandidate}_{i+1}(F^{i}_{new}, \mathcal{F}^{i}_{s})$
	\ENDFOR
\ENDIF
\RETURN $(\mathcal{F}^{*}_{s}, \mathcal{B}^{*}_{s})$
\end{algorithmic}
\end{algorithm}

For a window $W_{s}$ ending in the batch $B_{s}$, the set of output patterns can be obtained by picking the top-$k$ most frequent patterns from the set $\mathcal{F}_{s}^{\ell}$. Each pattern also maintains a list that stores its batch-wise counts is last $m$ batches. The window frequency is obtained by adding these entries together. The output patterns are listed in decreasing order of their window counts.

\begin{example}
In this example we illustrate the procedure for incrementally updating the frequent patterns lattice as a new batch $B_{s}$ is processed (see Figure~\ref{fig:lattice}).
\begin{figure}[htbp]
\centering
\includegraphics[width=\columnwidth]{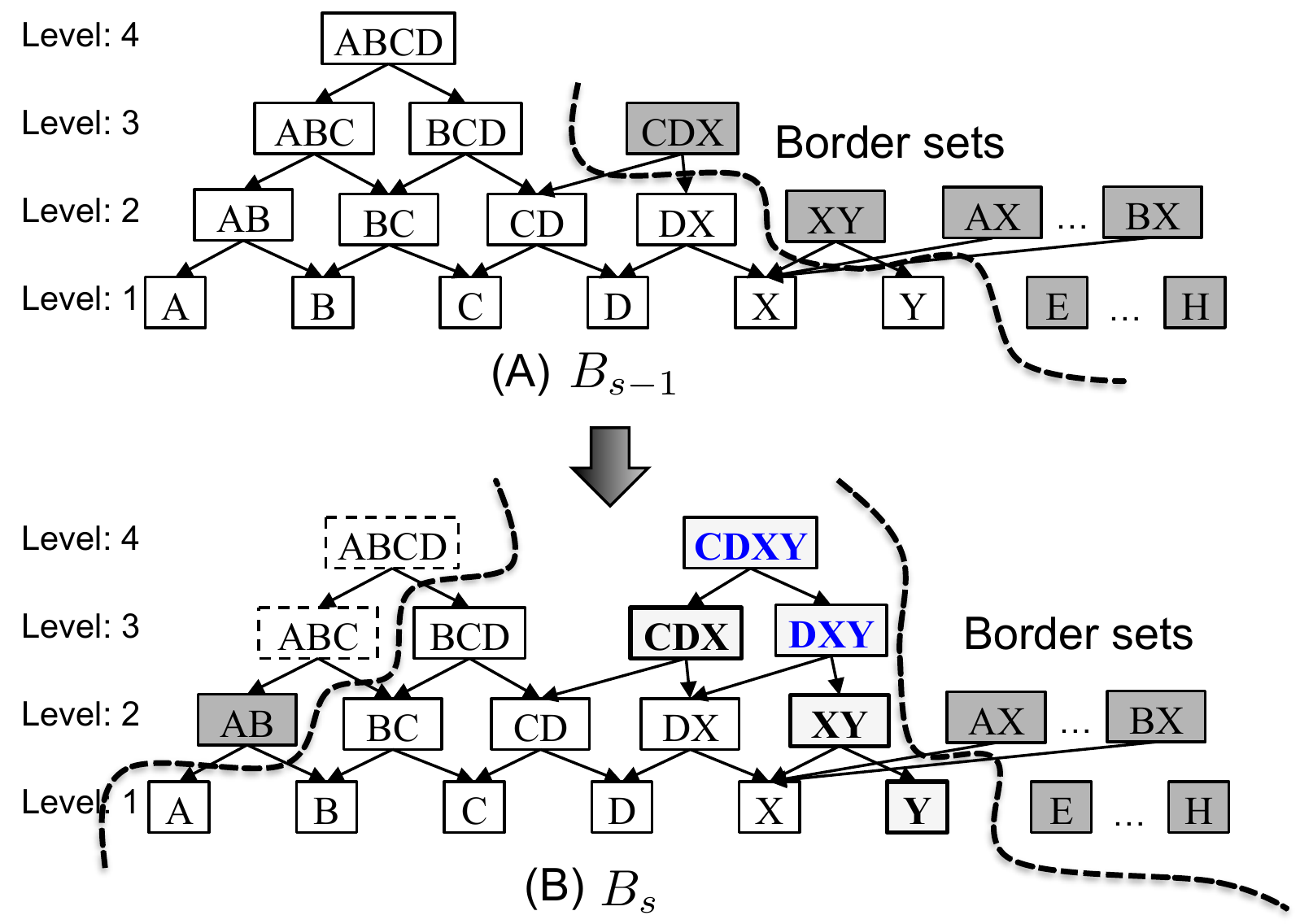}
\caption{Incremental lattice update for the next batch $B_{s}$ given the lattice of frequent and border patterns in $B_{s-1}$.}
\label{fig:lattice}
\end{figure}

Figure~\ref{fig:lattice}(A) shows the lattice of frequent and border patterns found in the batch $B_{s-1}$. $A B C D$ is a 4-size frequent pattern in the lattice. In the new batch $B_{s}$, the pattern $A B C D$ is no longer frequent. The pattern $C D X Y$ appears as a new frequent pattern. The pattern lattice in $B_{s}$ is shown in Figure~\ref{fig:lattice}(B).

In the new batch $B_{s}$, $A B$ falls out of the frequent set. $AB$ now becomes the new border and all its super-patterns namely $ABC$, $BCD$ and $ABCD$ are deleted from the lattice.

At level 2, the border pattern $X Y$ turns frequent in $B_{s}$. This allows us to generate $DXY$ as a new 3-size candidate. At level 3, $DXY$ is also found to be frequent and is combined with $CDX$ which is also frequent in $B_{s}$ to generate $CDXY$ as a 4-size candidate. Finally at level 4, $CDXY$ is found to be frequent. This shows that border sets can be used to fill out the parts of the pattern lattice that become frequent in the new data.

\end{example}

\subsection{Estimating $\Delta$ dynamically}
The parameter $\Delta$ in the bounded rate change assumption is a critical parameter in the entire formulation. But unfortunately the choice of the correct value for $\Delta$ is highly data-dependent. In the streaming setting, the characteristics of the data can change over time. Hence one predetermined value of $\Delta$ cannot be provided in any intuitive way. Therefore we estimate $\Delta$ from the frequencies of $\ell$-size episodes in consecutive windows. We compute the differences in frequencies of episodes that are common in consecutive batches. Specifically, we consider the value at the 75th percentile as an estimate of $\Delta$. We avoid using the maximum change as it tends to be noisy. A few patterns exhibiting large changes in frequency can skew the estimate and adversely affect the mining procedure.

\section{Results}
\label{sec:results}

In this section we present results both on synthetic data and data from real neuroscience experiments. We compare the performance of the proposed streaming episode mining algorithm on synthetic data to quantify the effect of different parameter choices and data characteristics on the quality of the top-k episodes reported by each method. Finally we show the quality of results obtained on neuroscience data.

For the purpose of comparing the quality of results we setup the following six variants of the mining frequent episodes:
\begin{description}
\item[Alg 0:] This is the naive brute force top-k mining algorithm that loads an entire window of events at a time and mines the top-k episode by repeatedly lowering the frequency threshold for mining. When a new batch arrives, events from the oldest batch are retired and mining process is repeated from scratch. This method acts as the baseline for comparing all other algorithms in terms of precision and recall.

\item[Alg 1:] The top-k mining is done batch-wise. The top-k episodes over a window are 
reported from within the set of episodes that belong to the batch-wise top-k of atleast one batch in the window.

\item[Alg 2:] Here the algorithm is same is above, but once an episodes enters  the top-k in any of the batches in a window, it is tracked over several subsequent batches. An episode is removed from the list of episodes being tracked if it does not occur in the top-k of last $m$ consecutive batches. This strategy helps obtaining a larger candidate set and also in getting more accurate counts of candidate patterns over the window.

\item[Alg 3:] This algorithm uses a batch-wise frequency threshold $\fk{s}-2\delta$ which ensures that the top-k episodes in the next batch $B_{s+1}$ are contained in the frequent lattice of $B_{s}$. This avoids multiple passes of the data while trying to obtain $k$ most frequent episodes lowering the support threshold iteratively. The patterns with frequency between $\fk{s}$ and $\fk{s}-2\delta$ also improve the overall precision and recall with respect to the window.

\item[Alg 4:] In this case the batch-wise frequency threshold is $\fk{s}-2(m-v)\delta$ which guarantees finding all $(v,k)$-persistent episodes in the data. We report results for $v=3m/4$ and $v=m/2$.

\item[Alg 5:] Finally, this last algorithm uses a heuristic batchwise threshold of $f_{k}^s -m(2-\frac{v}{m} -(\frac{v}{m})^2)\Delta$. Again we report results for $v=3m/4$ and $v=m/2$.

\end{description}

\subsection{Synthetic Datasets}
The datasets we used for experimental evaluation are listed in Table~\ref{tab:datasets}. The name of the
data set is listed in Column 1, the length of the data set (or number of time-slices in the data
sequence) in Column 2, the size of the alphabet (or total number of event types) in
Column 3, the average rest firing rate in Column 4 and the number of patterns embedded in Column 5.
In these datasets the data length is varied from - million to - million events, the alphabet size is varied from 1000 to 5000, the resting firing rate from 10.0 to 25.0, 
and the number of patterns embedded in the data from 25 to 50.

\paragraph{Data generation model:}
The data generation model for synthetic data is based on the inhomogeneous Poisson process model for evaluating the algorithm for learning excitatory dynamic networks \cite{PLR10}. We introduce two changes to this model. First, in order to mimic real data more closely 
in the events that constitute the background noise the event-type distribution follows a power law distribution. This gives the long tail characteristics to the simulated data.

The second modification was to allow the rate of arrival of episodes to change over time. As time progresses, the frequency of episodes in the recent window or batch slowly changes. We use a randomized scheme to update the connection strengths in the neuronal simulation model. The updates happen at the same timescale as the batch sizes used for evaluation.

\begin{table}
  \centering
  \caption{Datasets}
  \label{tab:datasets}
  \vspace{1mm}
    \begin{tabular}{|l|rrr|}
	\hline
    Dataset & Alphabet  & Rest Firing & Number of \\
	Name	& Size & Rate	& Patterns \\
	\hline
	A1 & 500 & 10.0 & 50\\
	A2 & 1000 & 10.0 & 50\\
	A3 & 5000 & 10.0 & 50\\
	\hline
	B1 & 1000 & 2.0 & 50\\
	B2 & 1000 & 10.0 & 50\\
	B3 & 1000 & 25.0 & 50\\
	\hline
	C1 & 1000 & 10.0 & 10\\
	C1 & 1000 & 10.0 & 25\\
	C1 & 1000 & 10.0 & 50\\
	\hline
    \end{tabular}
\end{table}

\subsection{Comparison of algorithms}

In Fig.~\ref{fig:all-compare}, we compare the five algorithms---Alg 1 through Alg 5---that 
report the frequent episodes over the window looking at one batch at a time with the baseline algorithm Alg 0 that stores and processes the entire window at each window slide. The results are averaged over all 9 data sets shown in Table~\ref{tab:datasets}. We expect to marginalize the data characteristics and give a more general picture of each algorithm. The parameter settings for the experiments are shown in Table~\ref{tab:params}. 
Fig.~\ref{fig:all-compare} (a) plots the precision of the output of each algorithm compared to that of Alg 0 (treated as ground truth). Similarly, Fig~\ref{fig:all-compare} (b) shows the recall. Since the size of output of each algorithm is roughly $k$, the corresponding precision and recall numbers are almost the same.
Average runtimes are shown in Fig.~\ref{fig:all-compare}~(c) and average memory requirement in MB is shown in Fig.~\ref{fig:all-compare}~(d).

\begin{table}[htdp]
\centering
\caption{Parameter settings}
\label{tab:params}
\begin{tabular}{|l|l|}
\hline
Parameter & Value(s)\\
\hline
Batch size $T_{b}$ & $10^{5}$ sec ($\approx 1$ million events per batch)\\
Number of batches in a window $m$ & 10 (5,15)\\
$v$, in $(v,k)$-persistence & 0.5m, 0.75m\\
$k$ in $(v,k)$-persistence and in top-$k$ & 25, 50\\
$\ell$ - size of episode & 4 \\
\hline
\end{tabular}
\end{table}

\begin{figure}[htbp]
\centering
\subfigure[Precision] {\includegraphics[width=0.45\columnwidth,height=2in]{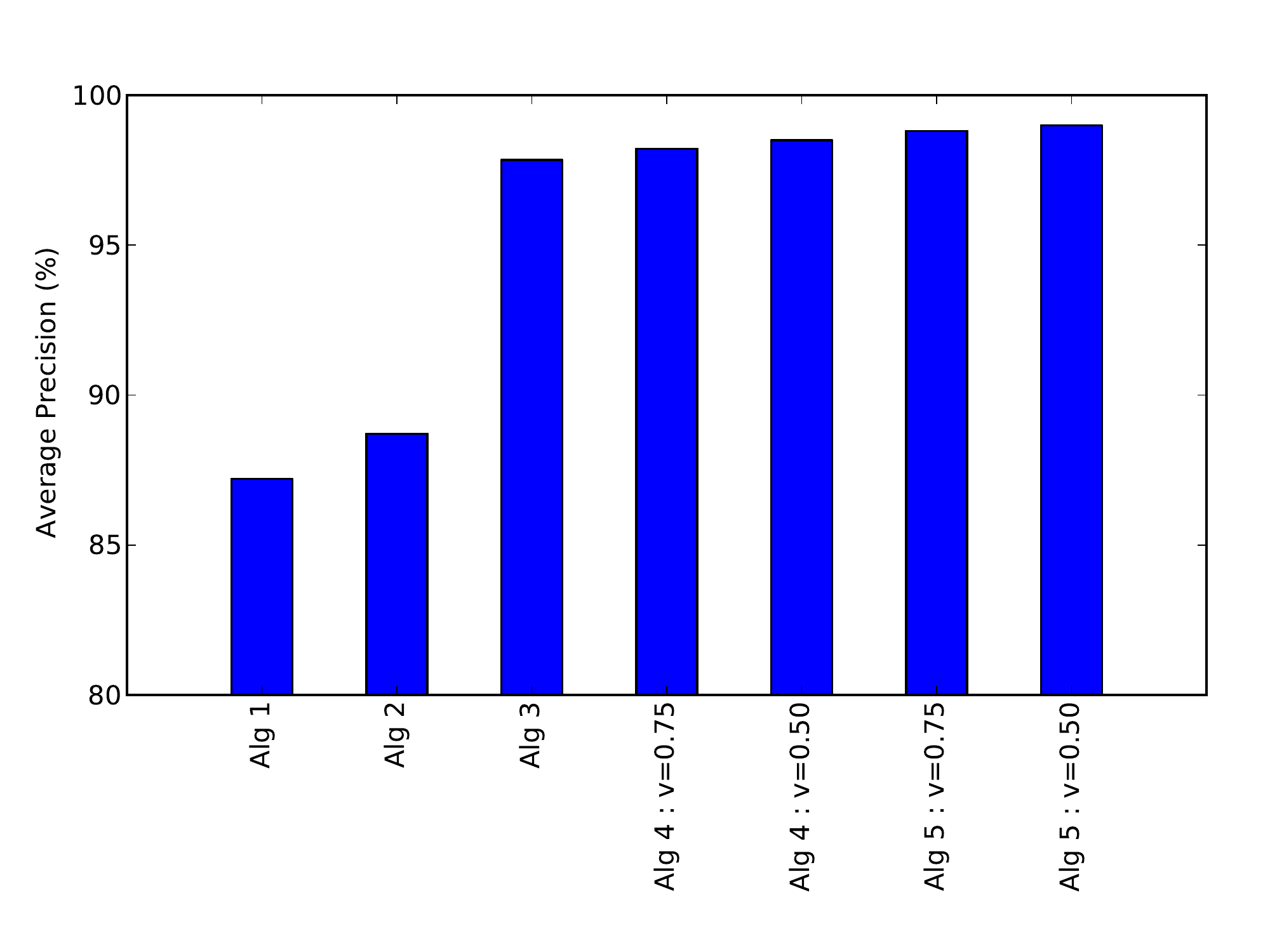}}
\subfigure[Recall] {\includegraphics[width=0.45\columnwidth,height=2in]{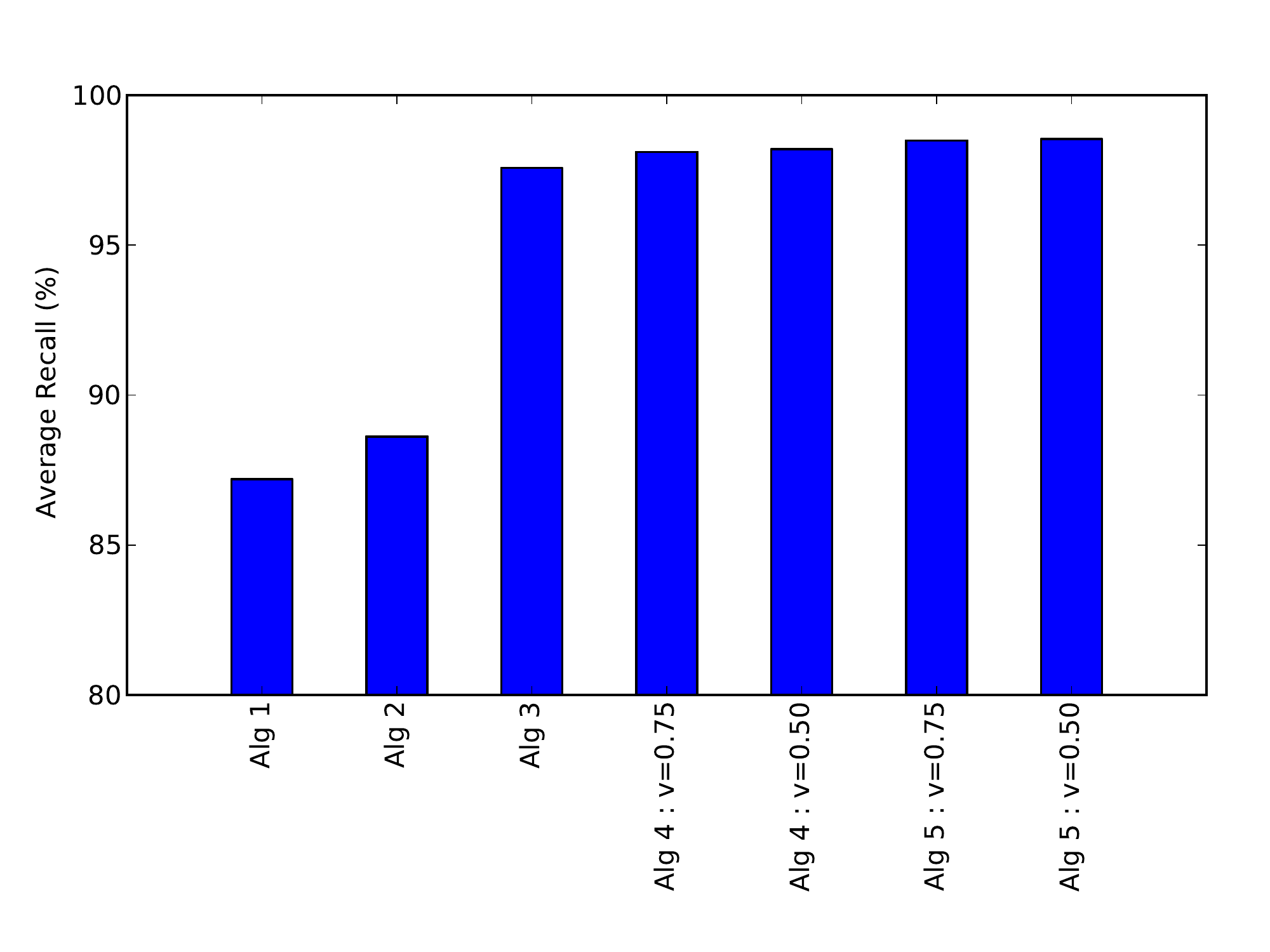}}
\subfigure[Runtime] {\includegraphics[width=0.45\columnwidth,height=2in]{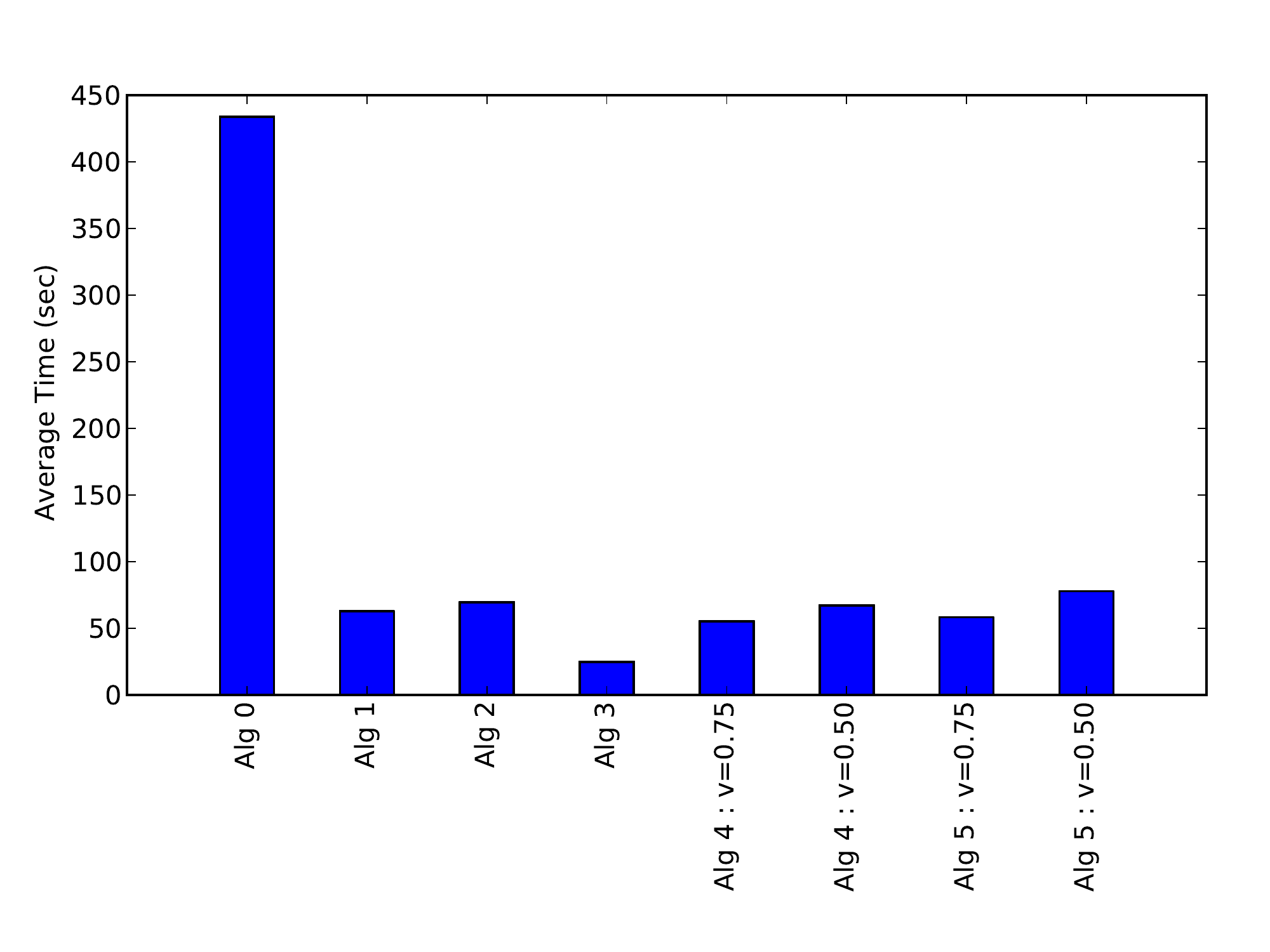}}
\subfigure[Memory] {\includegraphics[width=0.45\columnwidth,height=2in]{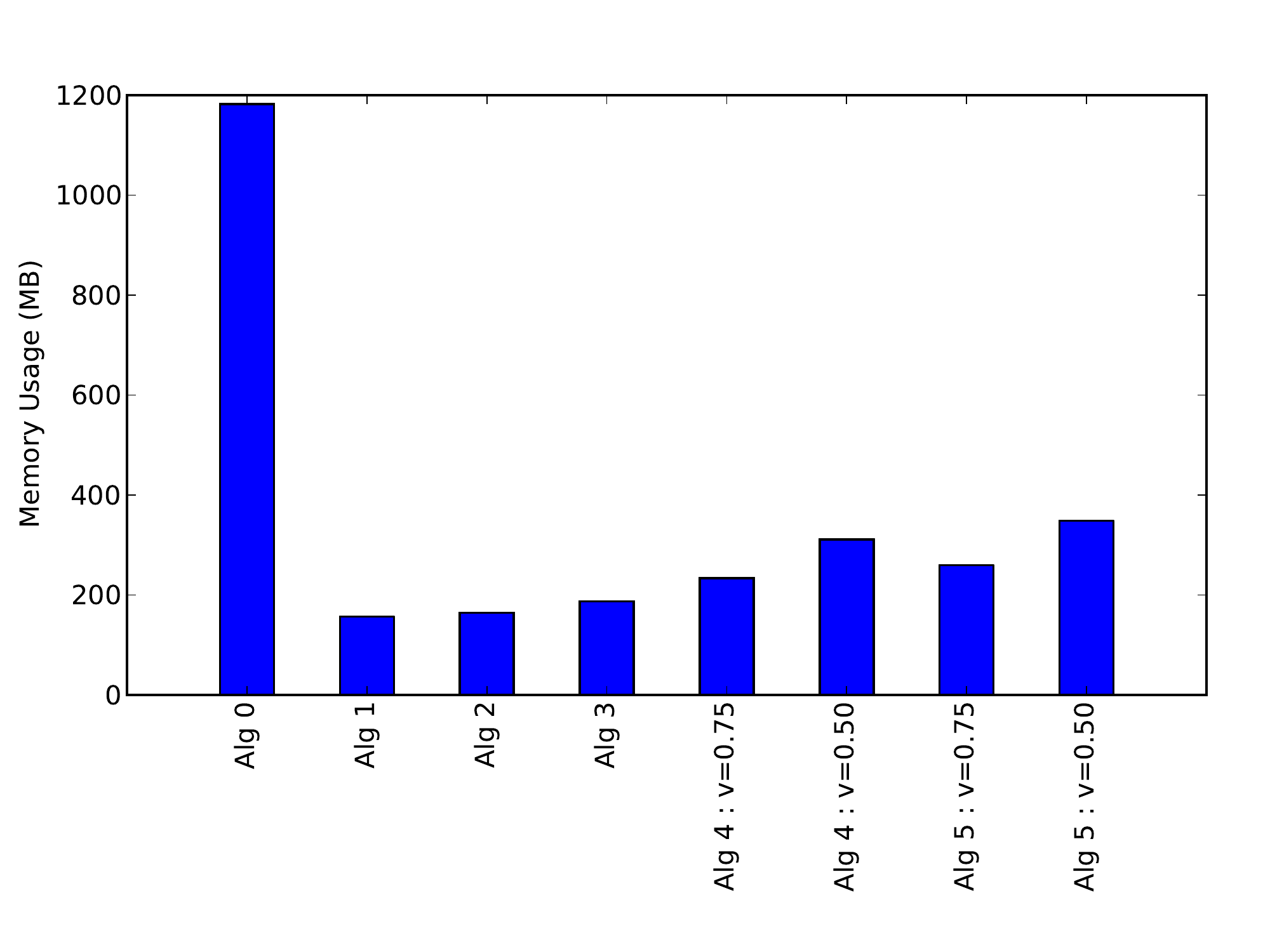}}
\caption{Comparison of average performance of different streaming episode mining algorithm. Alg 1 and 2 give lower precision and recall values compared with any other algorithms. Overall the proposed methods give atleast one order of magnitude improvement over the baseline algorithm (Alg 0) in terms of both time and space complexity.}
\label{fig:all-compare}
\end{figure}

We consistently observe that Alg 1 and 2 give lower precision and recall values compared with
any other algorithm. This reinforces our observation that the top-$k$ patterns in a window can be much different from the top-$k$ patterns in the constituent batches. Alg 2 provides only a slight improvement over Alg 1 by tracking an episode once it enters the top-$k$ over subsequent batches. This improvement can be attributed to the fact that window frequencies of patterns that were once in the top-$k$ is better estimated. Alg 3 gives higher precision and recall compared to Alg 1 and 2. The frequency threshold used in Alg 3 is given by Theorem~\ref{thm:topk-mine}. Using this result we are able to estimate the value $\fk{s}-2\delta$ by simply counting the episodes that are frequent in the previous batch. This avoids multiple iterations required in general for finding the top-$k$ patterns. Fortunately this threshold also results in significantly higher support and precision.

We ran Alg 4 and 5 for two different values of $v$, viz. $v = m/2$ and $v = 3m/4$. Both these algorithms guarantee finding all $(v,k)$-persistent patterns for respective values of $v$. For the sake of comparison we also include episodes that exceeded the frequency threshold prescribed by $(v,k)$-persistence, but do not appear in top-$k$ of atleast $v$ batches. We observe that the precision and recall improves a little over Alg 3 with reasonable increase in memory and runtime requirements. 
For a higher value of $v$, this algorithm insists that the frequent patterns much persist over more batches. This raises the support threshold and as a result there is improvement in terms of memory and runtime, but a small loss in precision and recall. Note that the patterns missed by the algorithm are either not $(v,k)$-persistent or our estimation of $\delta$ has some errors in it.
In addition, Alg 5 gives a slight improvement over Alg 4. This shows that our heuristic threshold is effective.

Overall the proposed methods give atleast one order of magnitude improvement over the baseline algorithm in terms of both time and space complexity.

\subsubsection{Performance over time}
Next we consider one dataset A2 with number of event types = 1000, the average resting firing rate as 10 Hz, and the number of embedded patterns = 50. On this data we show how the 
performances of the five algorithms change over time. The window size is set to be $m=10$ and the batch size = $10^{5}$ sec. Fig.~\ref{fig:over-time} shows the comparison over 50 contiguous batches. Fig.~\ref{fig:over-time} (a) and (b) show the way precision and recall evolve over time. Fig.~\ref{fig:over-time} (c) and (d) show the matching memory usage and runtimes.

The data generation model allows the episode frequencies to change slowly over time. In the 
dataset used in the comparison we change the frequencies of embedded episodes at two time intervals: batch 15 to 20 and batch 35 to 42. In Fig.~\ref{fig:over-time} (a) and (b), we have a special plot shown by the dashed line. This is listed as Alg 0 in the legend. What this line shows is the comparison of top-$k$ episodes between consecutive window slides. In other words the top-$k$ episodes in window $W_{s-1}$ are considered as the predicted output to obtain the precision and recall for $W_{s}$. The purpose of this curve is to show how the true top-$k$ set changes with time and show how well the proposed algorithms track this change.

\begin{figure}[!ht]
\centering
\subfigure[Precision] {\includegraphics[width=\columnwidth,height=1.4in]{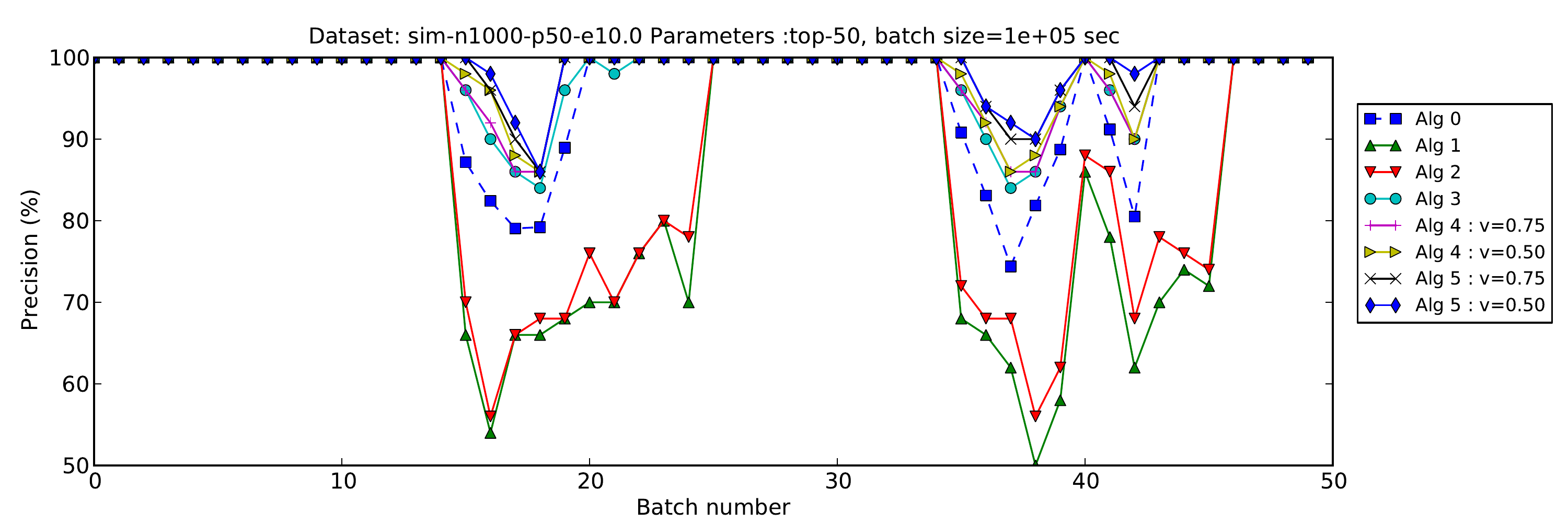}}
\subfigure[Recall] {\includegraphics[width=\columnwidth,height=1.4in]{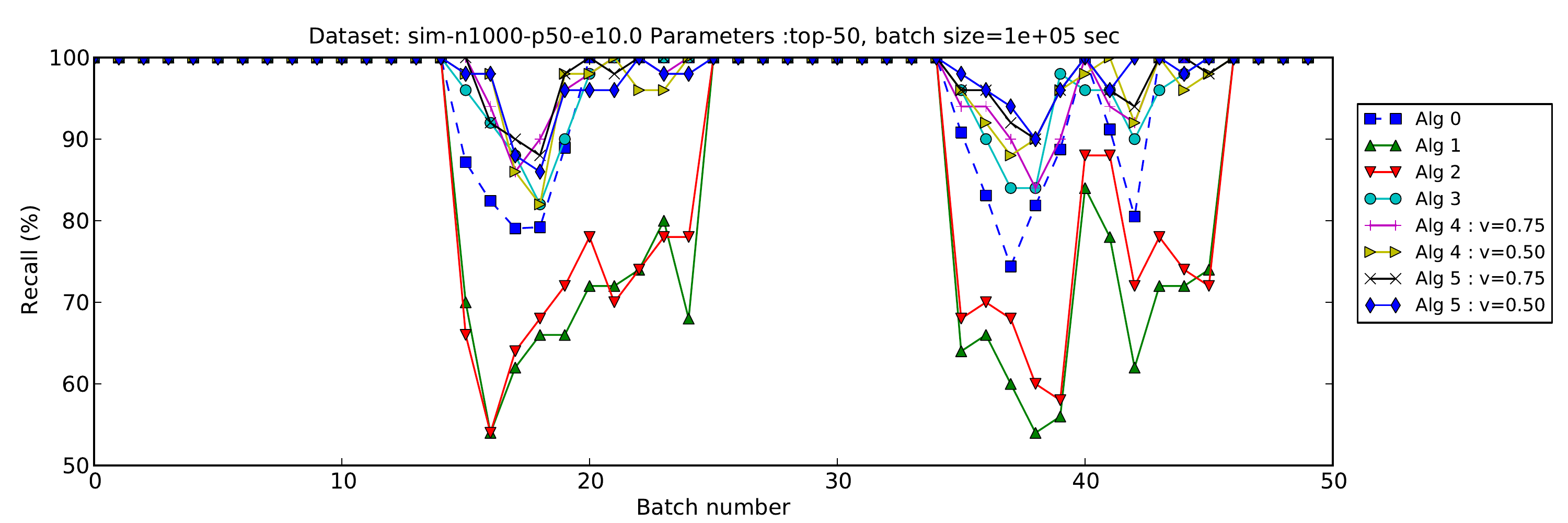}}
\subfigure[Runtime] {\includegraphics[width=\columnwidth,height=1.4in]{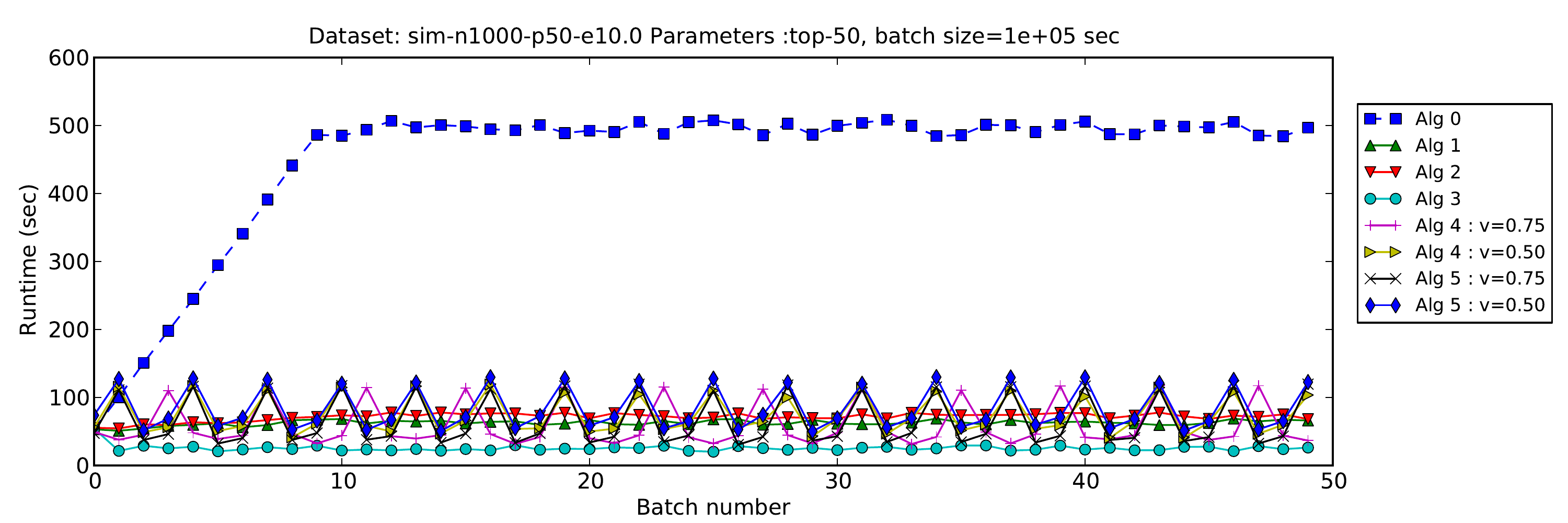}}
\subfigure[Memory] {\includegraphics[width=\columnwidth,height=1.4in]{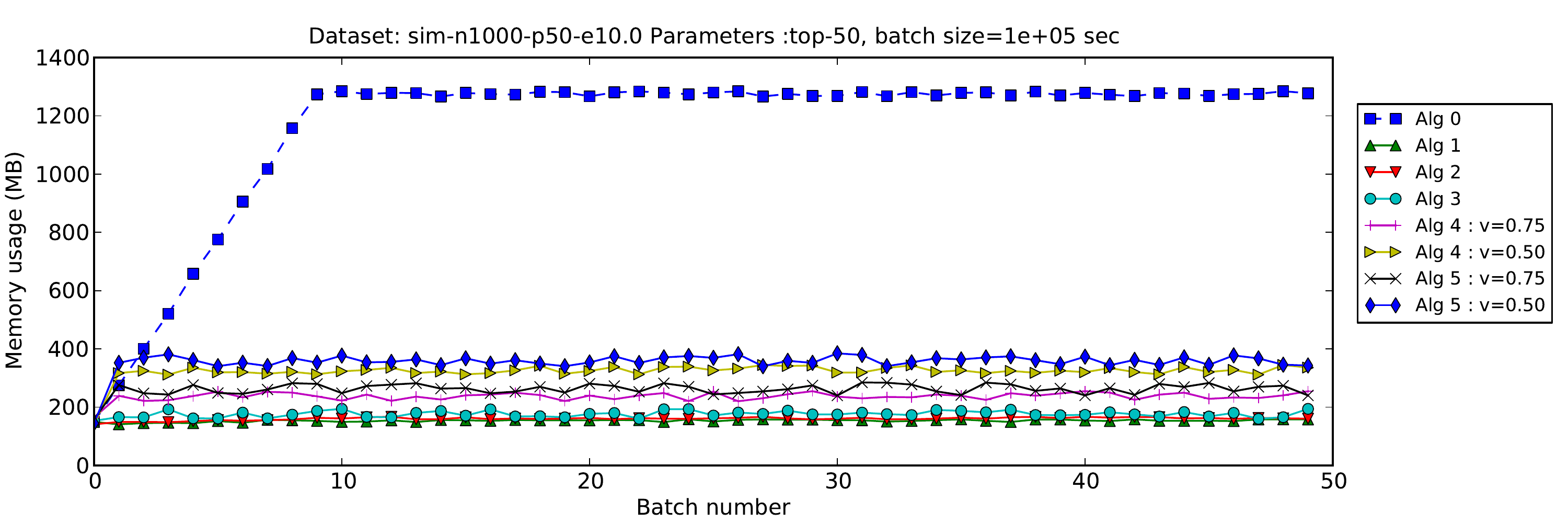}}
\caption{Comparison of the performance of different streaming episode mining algorithms over a sequence of 50 batches (where each batch is $10^{5}$ sec wide and each window consists of 10 batches).}
\label{fig:over-time}
\end{figure}

Alg 1 and 2 perform poorly. On an average in the transient regions (batch 15 to 20 and batch 35 to 42) they perform 15 to 20\% worse than any other method. Alg 3, 4 and 5 (for v=0.75m and v=0.5m) perform consistently above the reference curve of Alg 0. It expected of any reasonable algorithm to do better than the algorithm which uses the top-$k$ of $W_{s-1}$ to predict the top-$k$ of the window $W_{s}$. The precision and recall performance are in the order Alg 3 $<$ Alg 4 v=0.75m $<$ Alg 4 v=0.5m $<$ Alg 5 v=0.75m $<$ Alg 4 v=0.5m. This is in the same order as the frequency thresholds used by each method, and as expected.

In terms of runtime and memory usage, the changing top-$k$ does not affect these numbers. The lowest runtimes are those of Alg 3. The initial slope in the runtimes and memory usage seen in Algo 0, is due to the fact that the algorithm loads the entire window, one batch at a time into memory. In this experiment the window consists of $m=10$ batches. Therefore only after the first 10 batches one complete window span is available in memory.

\subsubsection{Effect of Data Characteristics}
In this section we present results on synthetic data with different characteristics, namely, number of event types (or alphabet size), noise levels and number of patterns embedded in the data.
\begin{figure}[!ht]
\centering
\subfigure[Precision] {\includegraphics[width=0.45\columnwidth,height=1.8in]{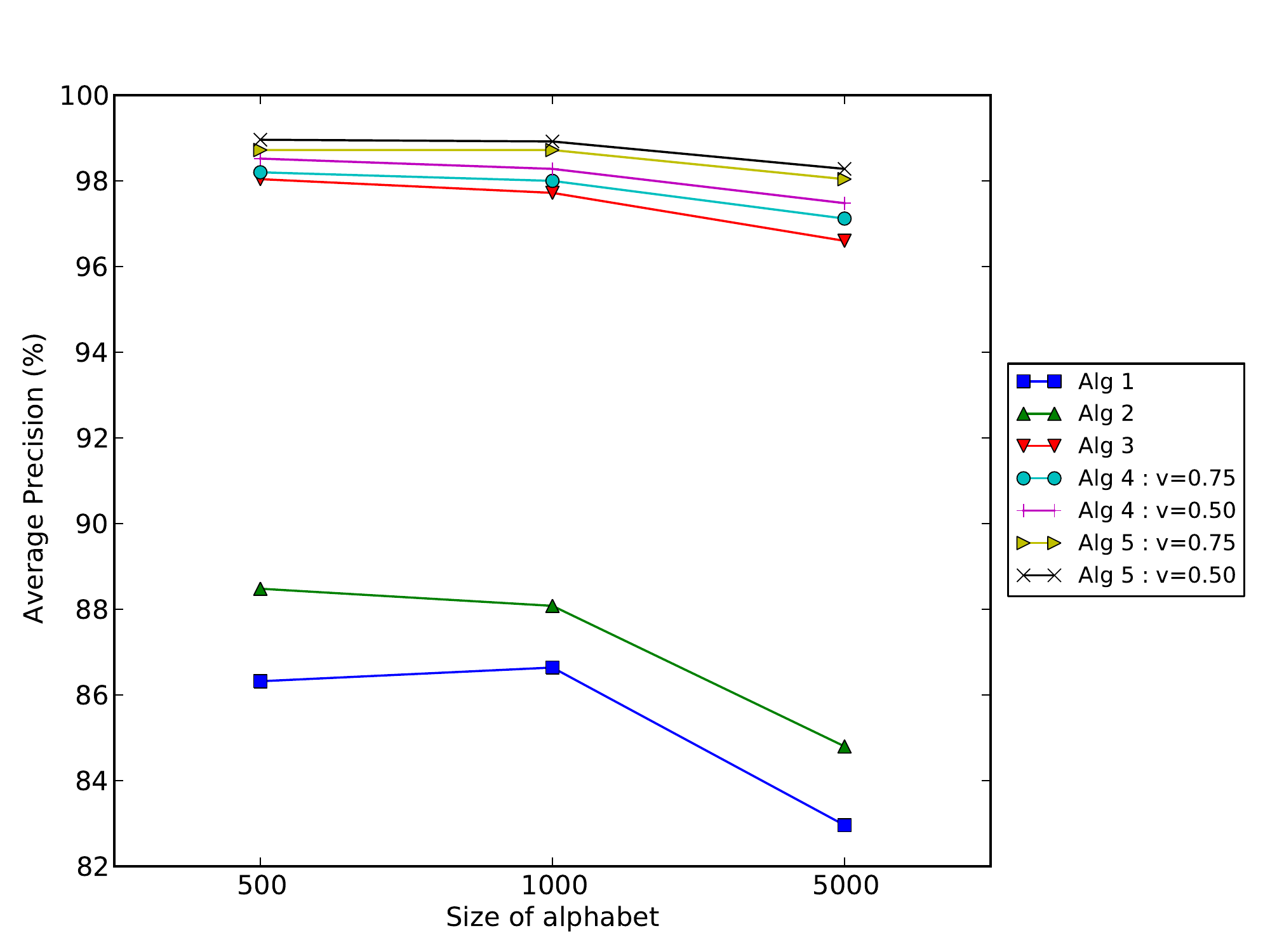}}
\subfigure[Recall] {\includegraphics[width=0.45\columnwidth,height=1.8in]{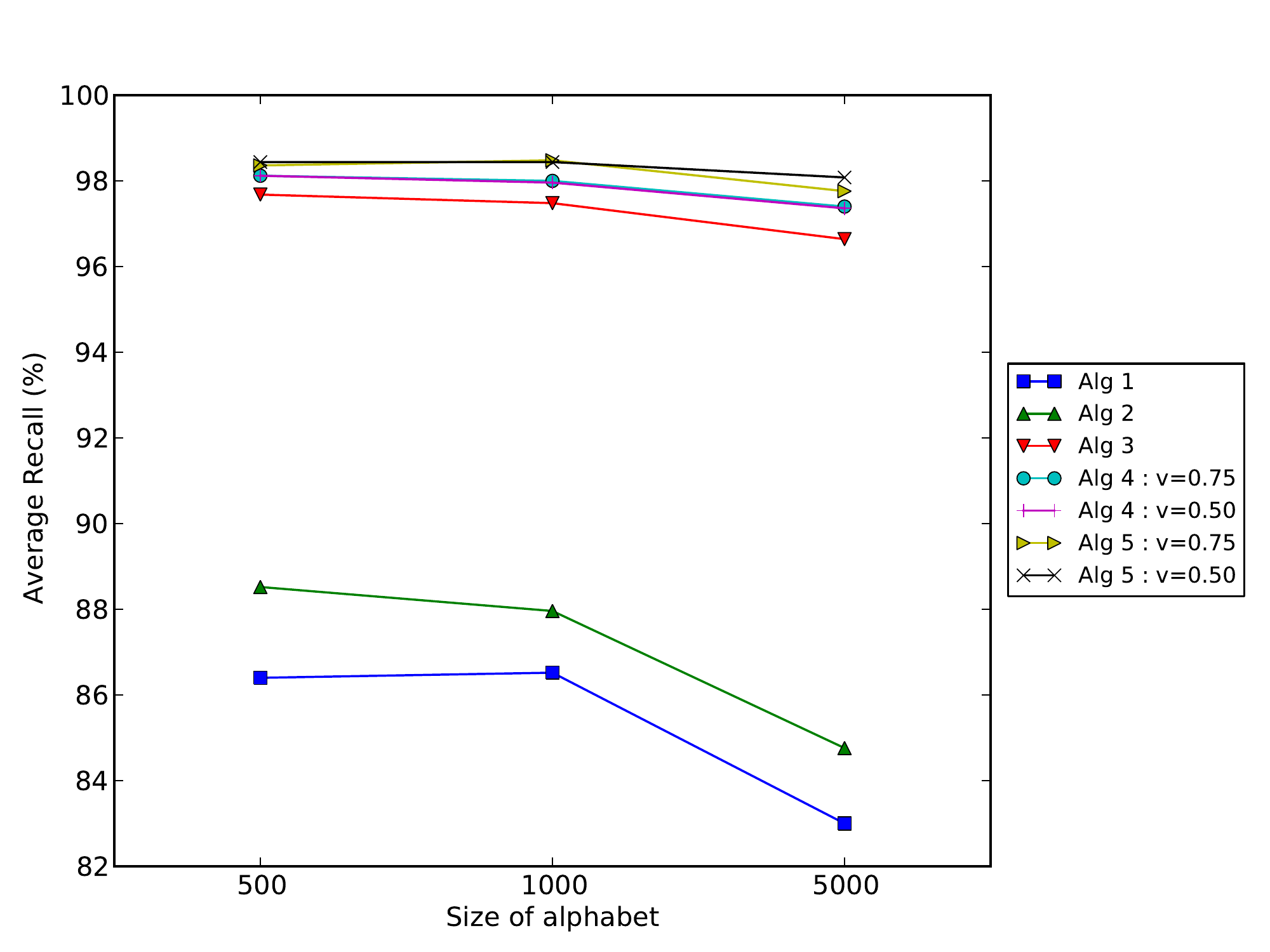}}
\subfigure[Runtime] {\includegraphics[width=0.45\columnwidth,height=1.8in]{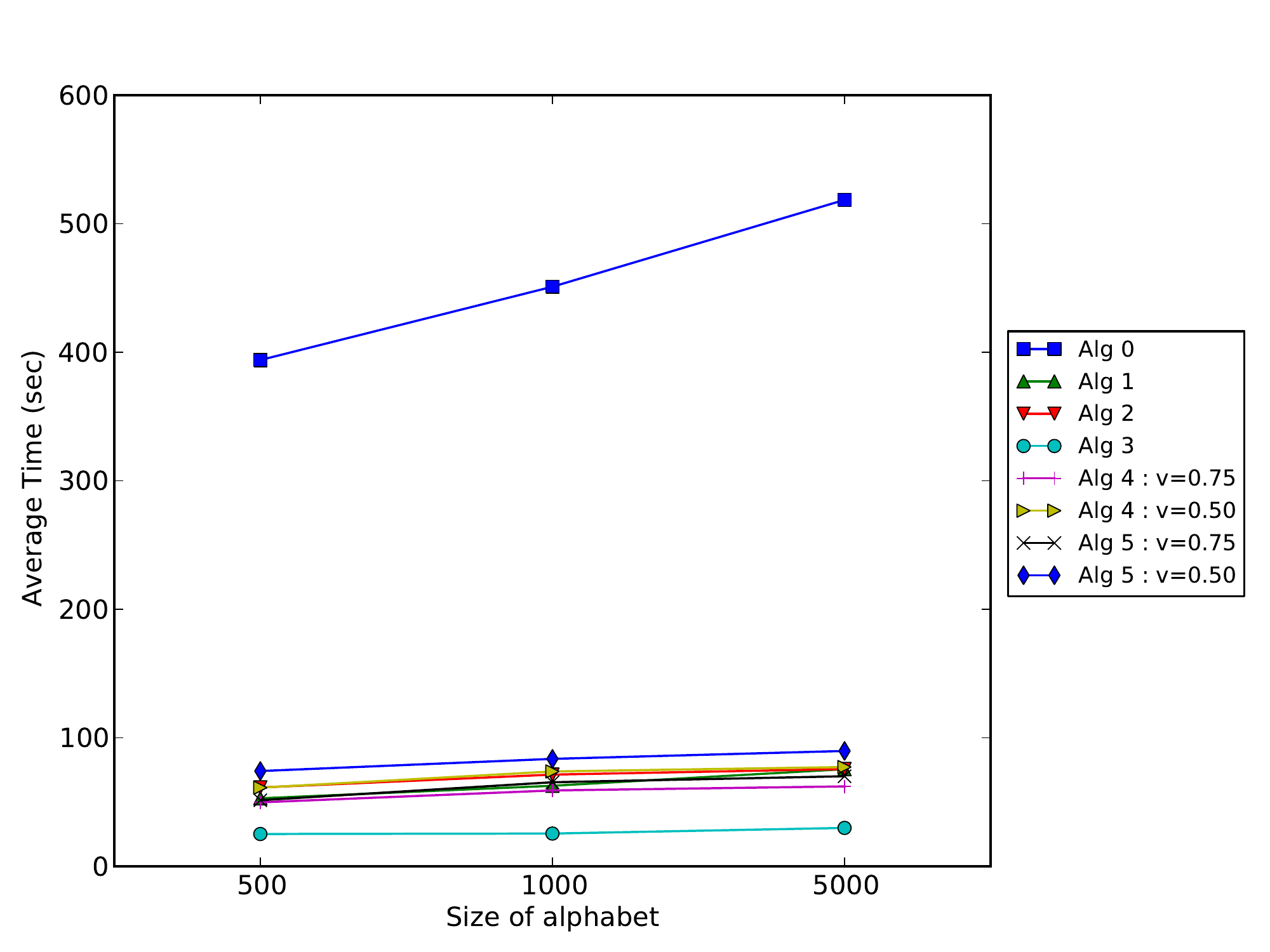}}
\subfigure[Memory] {\includegraphics[width=0.45\columnwidth,height=1.8in]{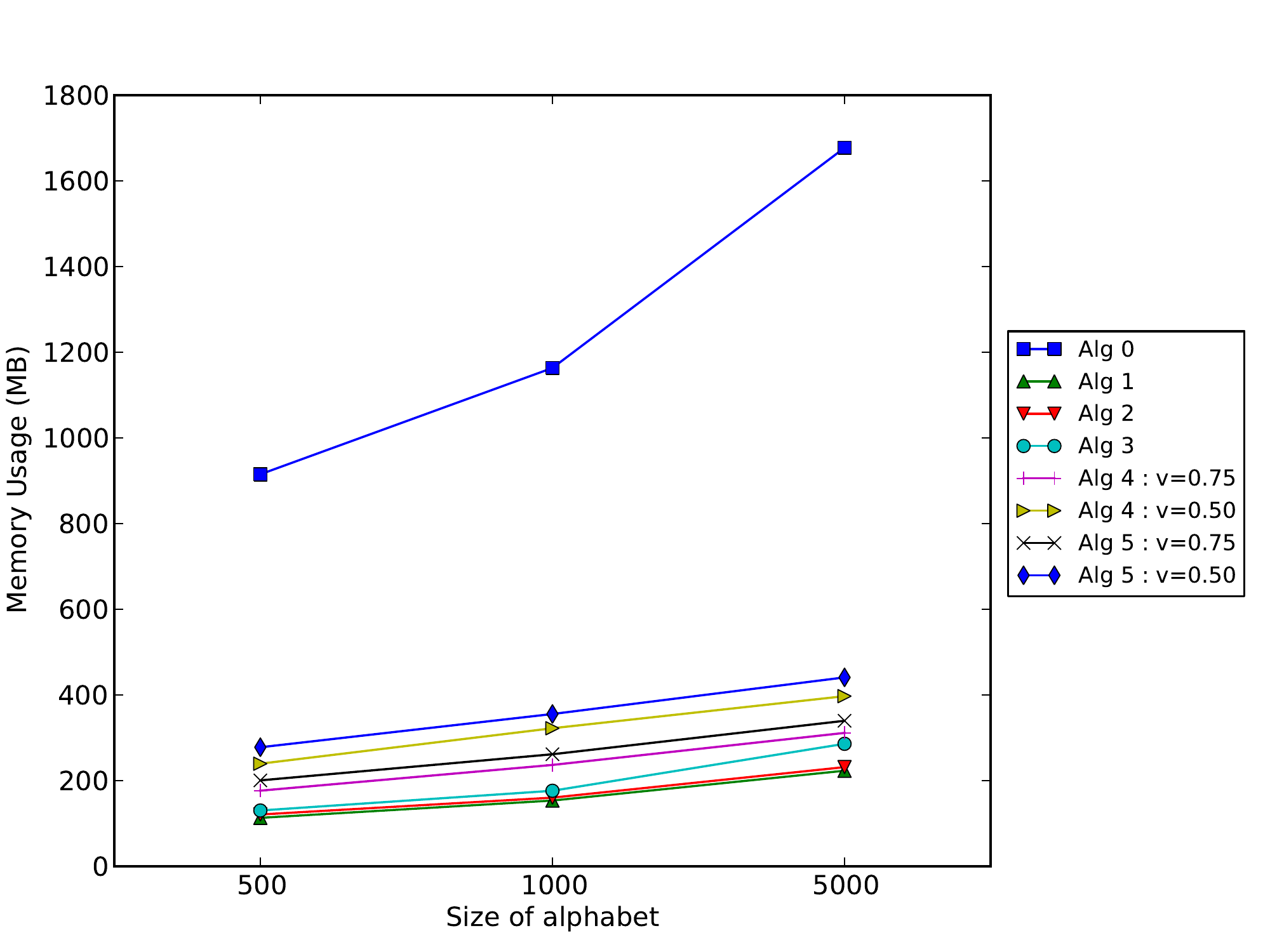}}
\caption{Effect of alphabet size. The proposed algorithms are robust to large alphabet sizes. Precision and recall drop by only 2-4\% going from alphabet size of 500 to 5000.}
\label{fig:alphabet}
\end{figure}

In Fig.~\ref{fig:alphabet} we report the effect alphabet size on the quality of result of the different algorithms. In datasets A1, A2 and A3 the alphabet size, i.e. the number of distinct event types, is varied from 500 to 5000. We observe that for smaller alphabet sizes the performance is better. Alg 1 and 2 perform consistently worse that the other algorithm for different alphabet sizes.

In this experiment we find that the quality of results for the proposed algorithms is not very sensitive to alphabet size. The precision and recall numbers drop by only 2-4\%. This is quite different from the pattern mining setting where the user provides a frequency threshold. In our experience alphabet size is critical in the fixed frequency threshold based formulation. For low thresholds, large alphabet sizes can quickly lead to uncontrolled growth in the number of candidates. In our formulation the support threshold is dynamically readjusted and as a result the effect of large alphabet size is attenuated.

\begin{figure}[t]
\centering
\subfigure[Precision] {\includegraphics[width=0.45\columnwidth,height=2in]{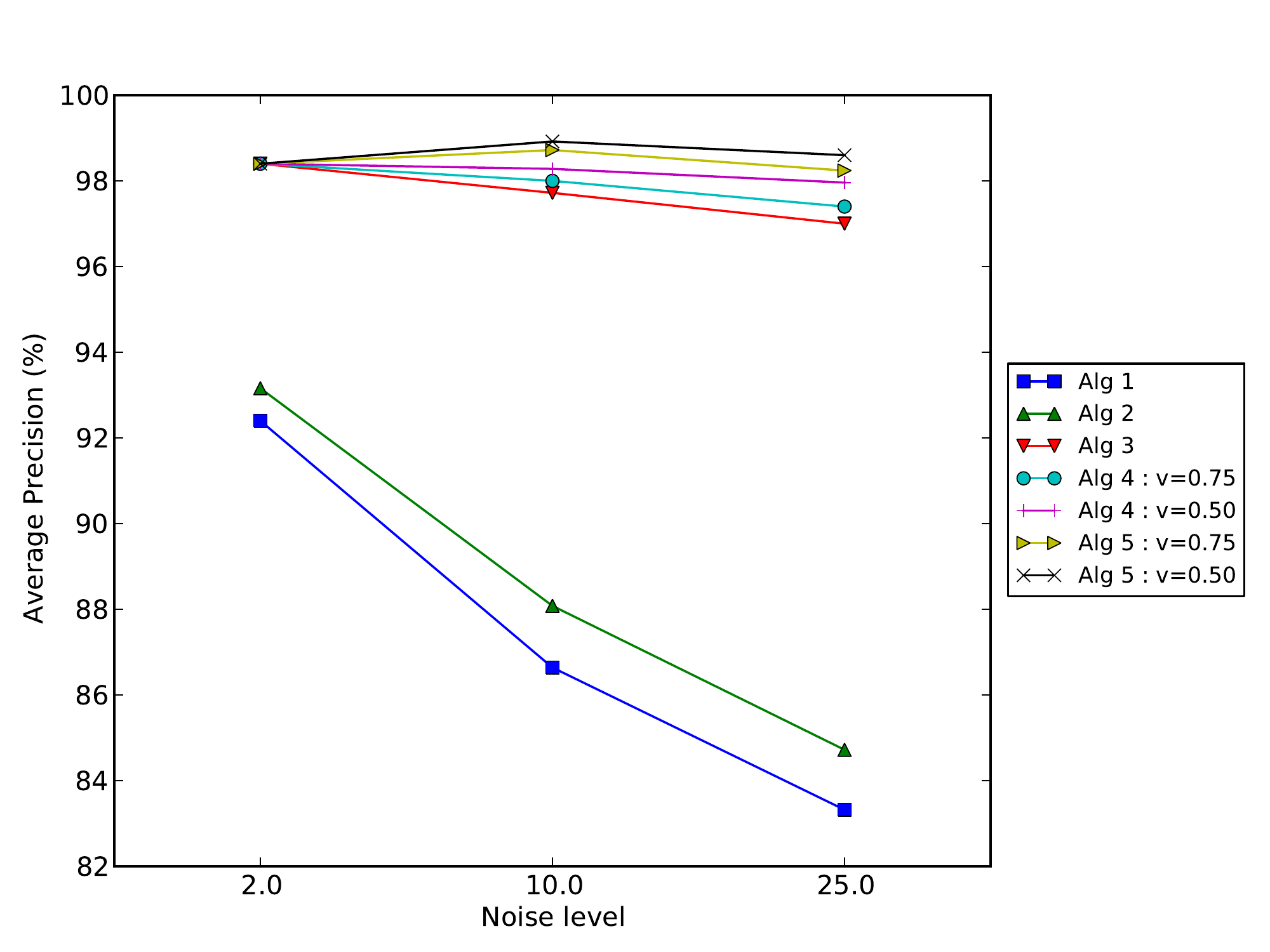}}
\subfigure[Recall] {\includegraphics[width=0.45\columnwidth,height=2in]{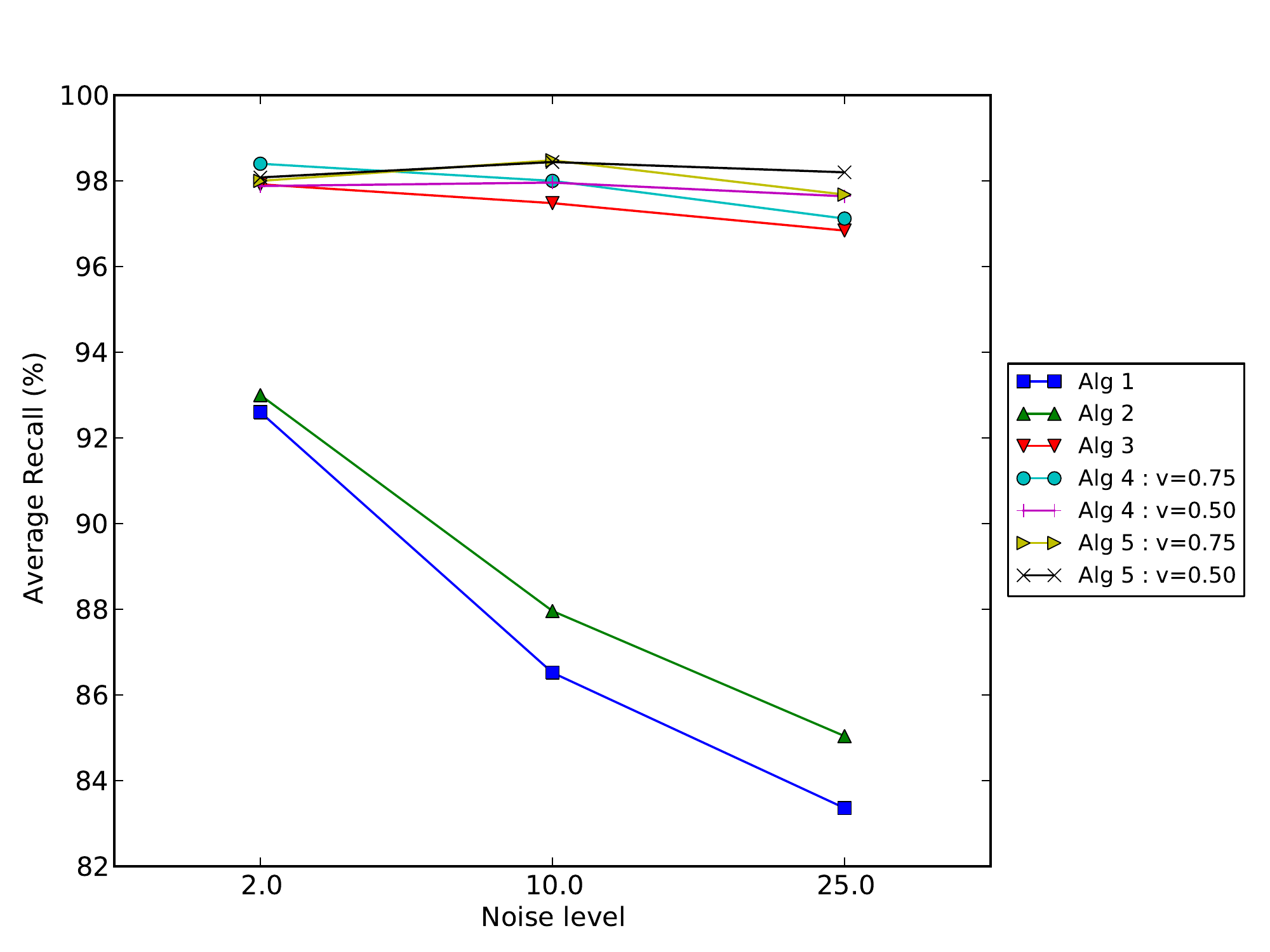}}
\subfigure[Runtime] {\includegraphics[width=0.45\columnwidth,height=2in]{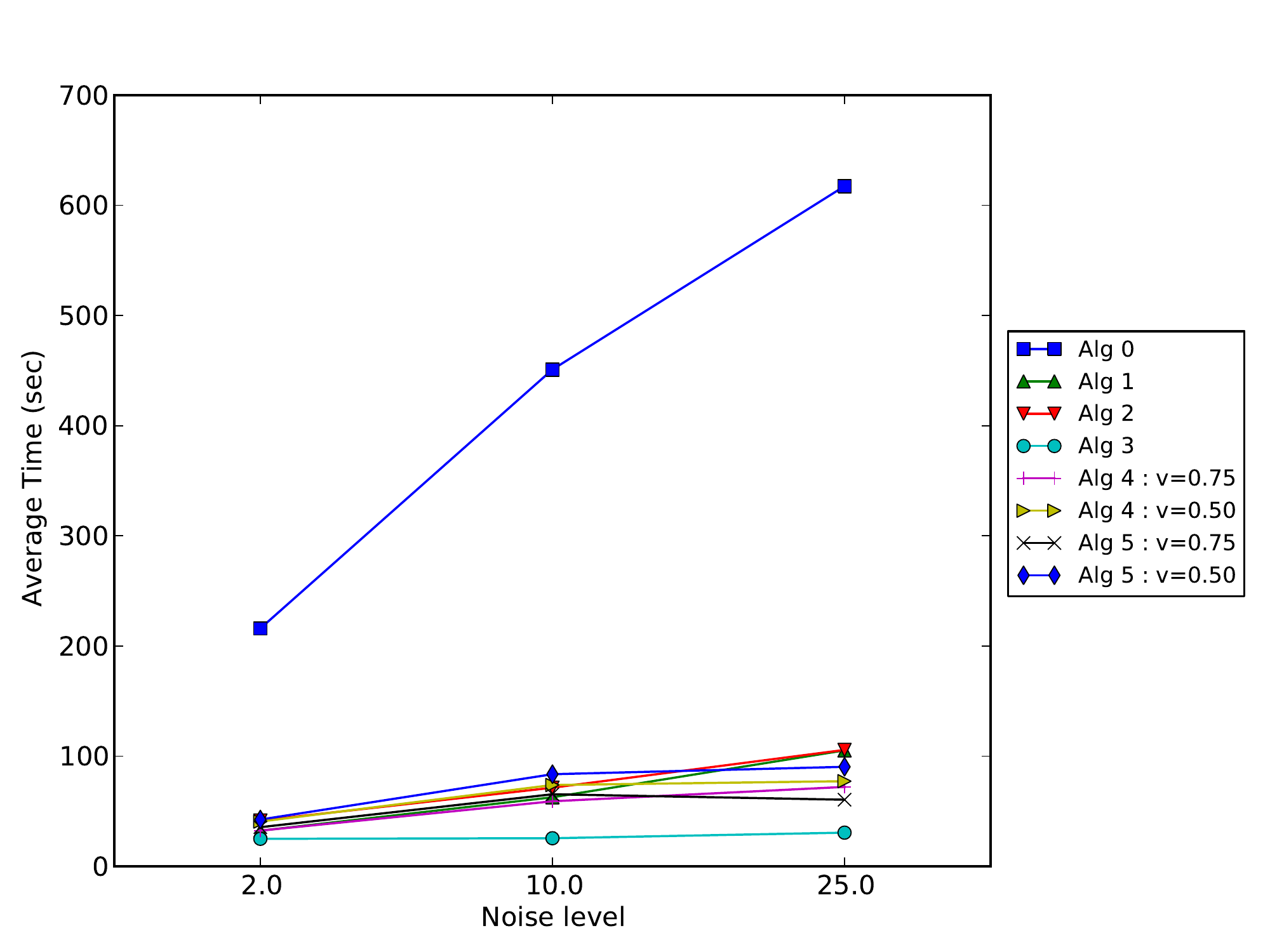}}
\subfigure[Memory] {\includegraphics[width=0.45\columnwidth,height=2in]{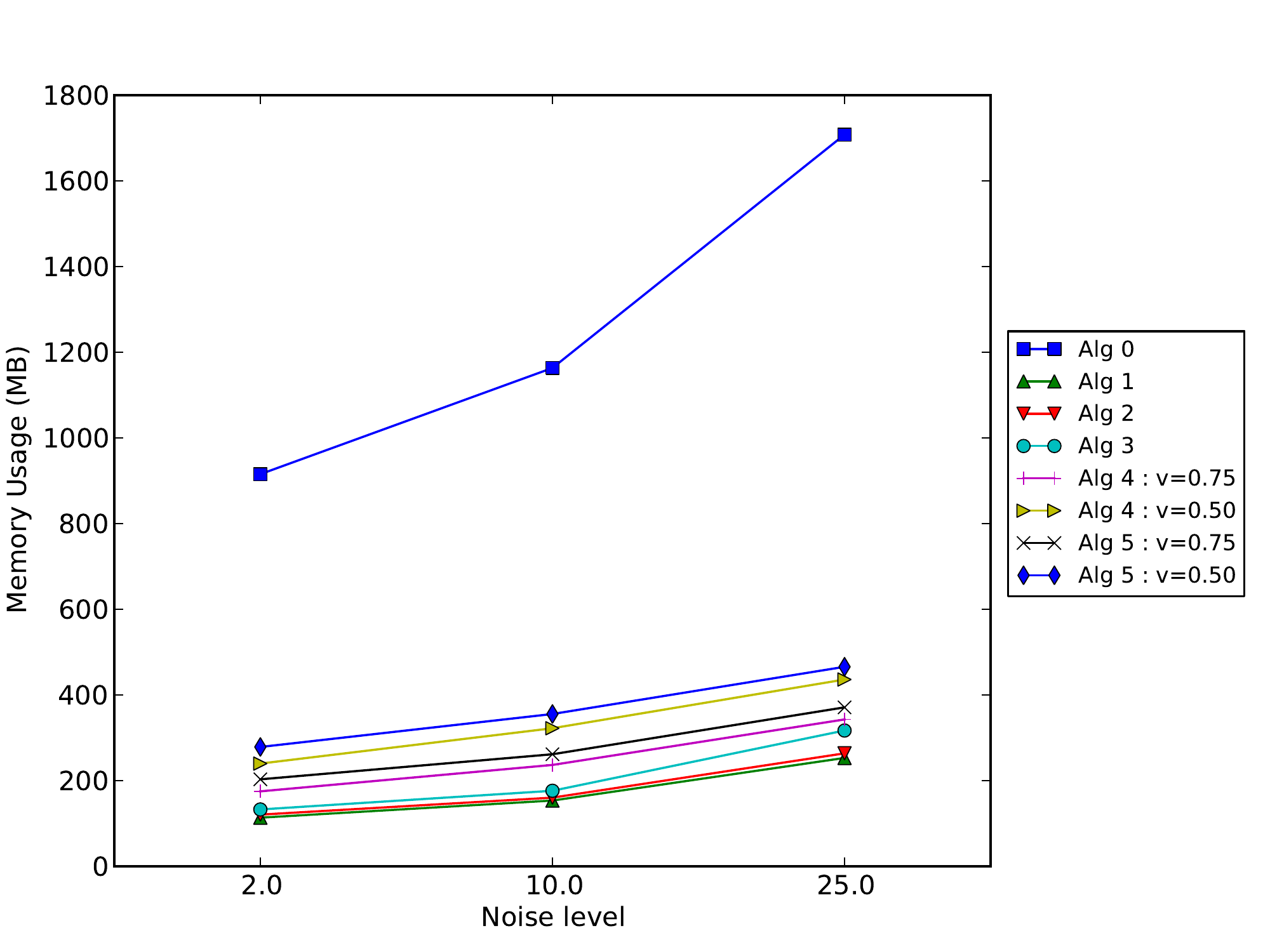}}
\caption{Effect of noise. In terms of precision and recall Alg 5 is most robust to noise in the data.}
\label{fig:noise}
\end{figure}

Next, in Fig.~\ref{fig:noise}, we show the effect of noise. The average rate of firing of the noise event-types (event types that do not participate in pattern occurrences) is varied  from 2.0 Hz to 25 Hz. The precision and recall of Alg 1 and 2 degrade quickly with increase in noise. A small decrease in precision and recall of Alg 3 and Alg 4 is seen. But the performance of Alg 5 (for both v=0.75m and v=0.5m) stays almost at the same level. It seems that the frequency threshold generated by Alg 5 is sufficiently low to finds the correct patterns even at higher noise level but not so low as to require significantly more memory ($\approx 400$ MB) or runtime ($\approx 70$ sec per batch at noise = 25.0 Hz for v=0.5m) as compared to other algorithms.

\begin{figure}[t]
\centering
\subfigure[Precision] {\includegraphics[width=0.45\columnwidth,height=1.8in]{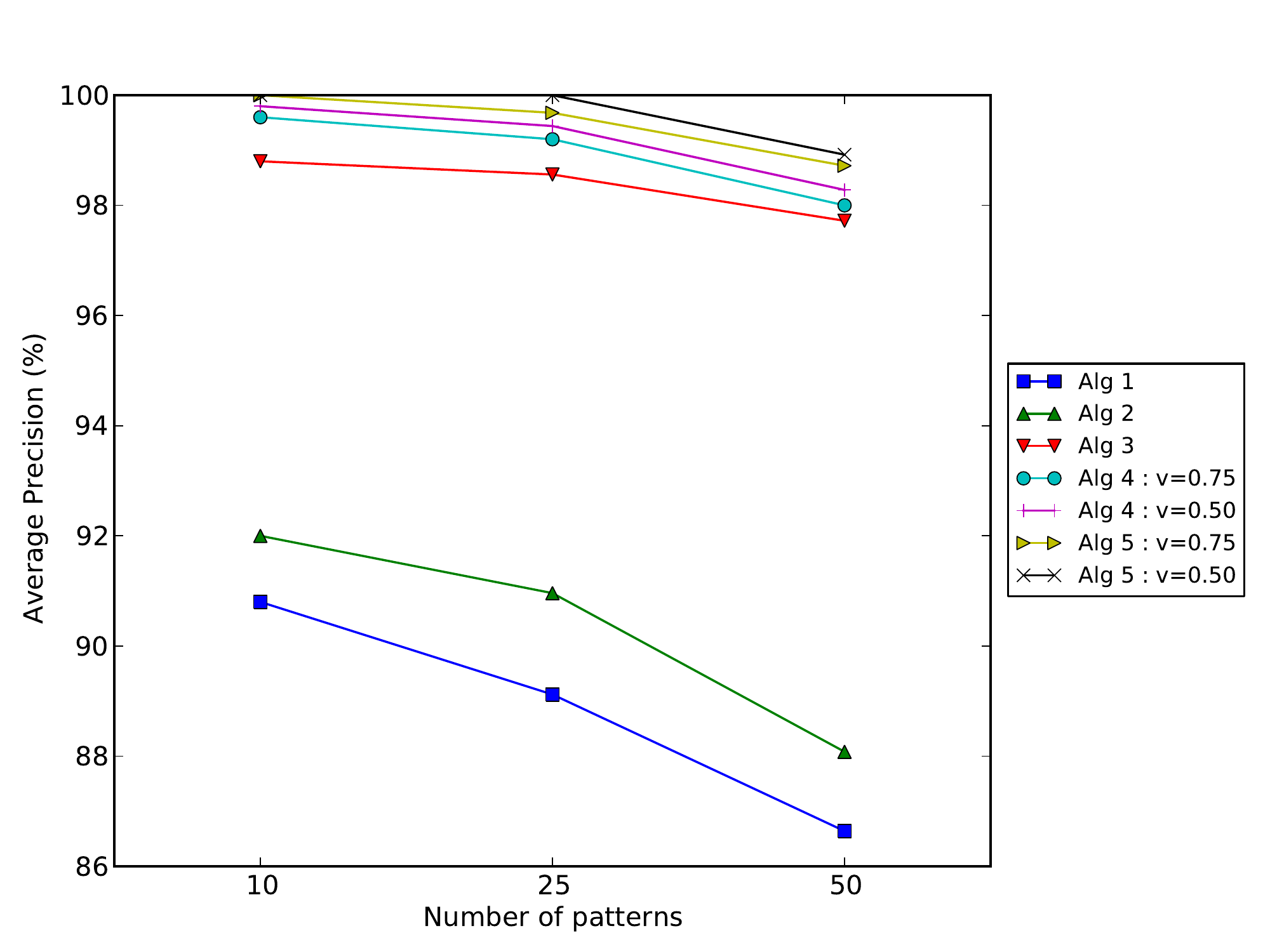}}
\subfigure[Recall] {\includegraphics[width=0.45\columnwidth,height=1.8in]{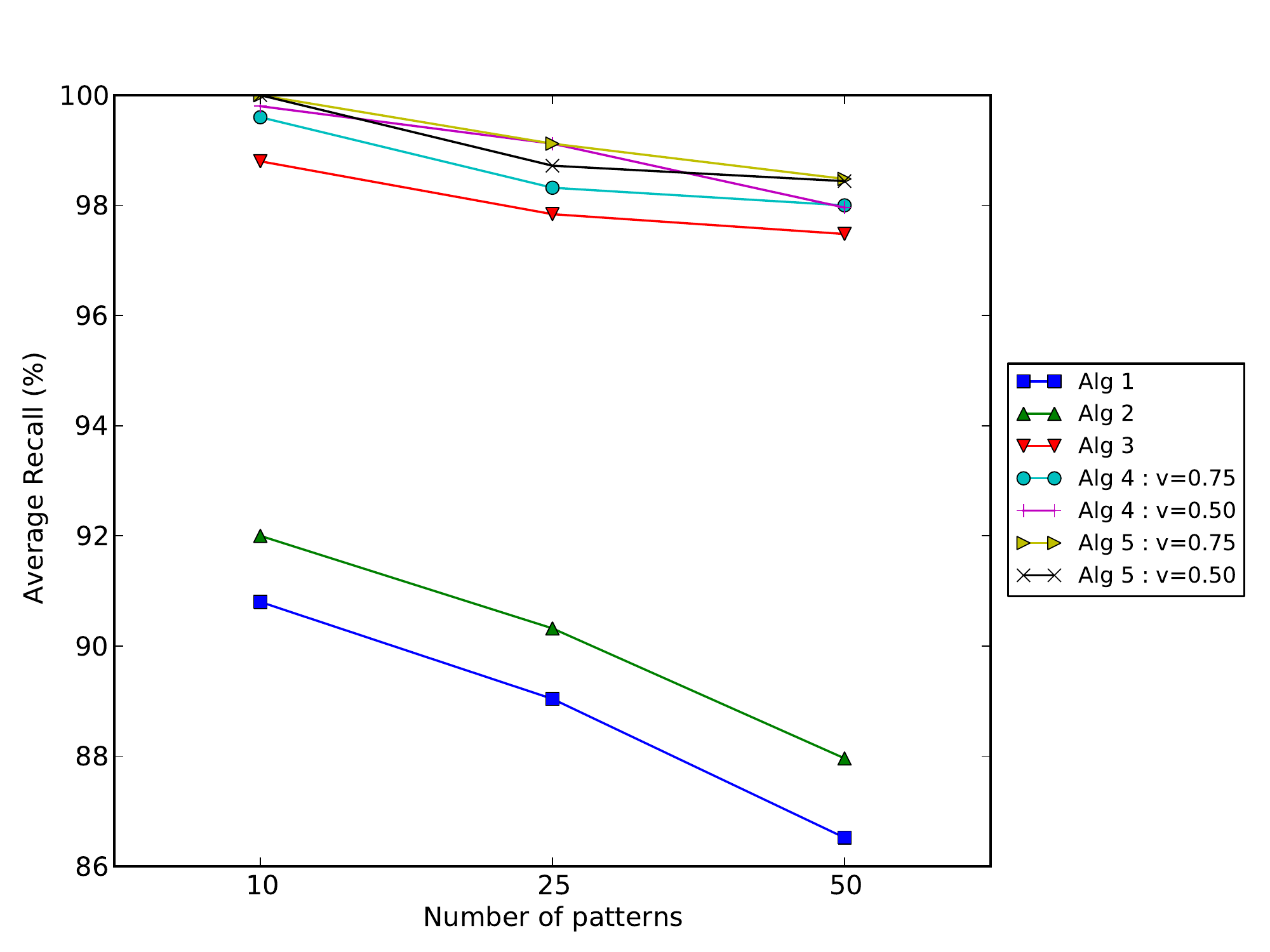}}
\subfigure[Runtime] {\includegraphics[width=0.45\columnwidth,height=1.8in]{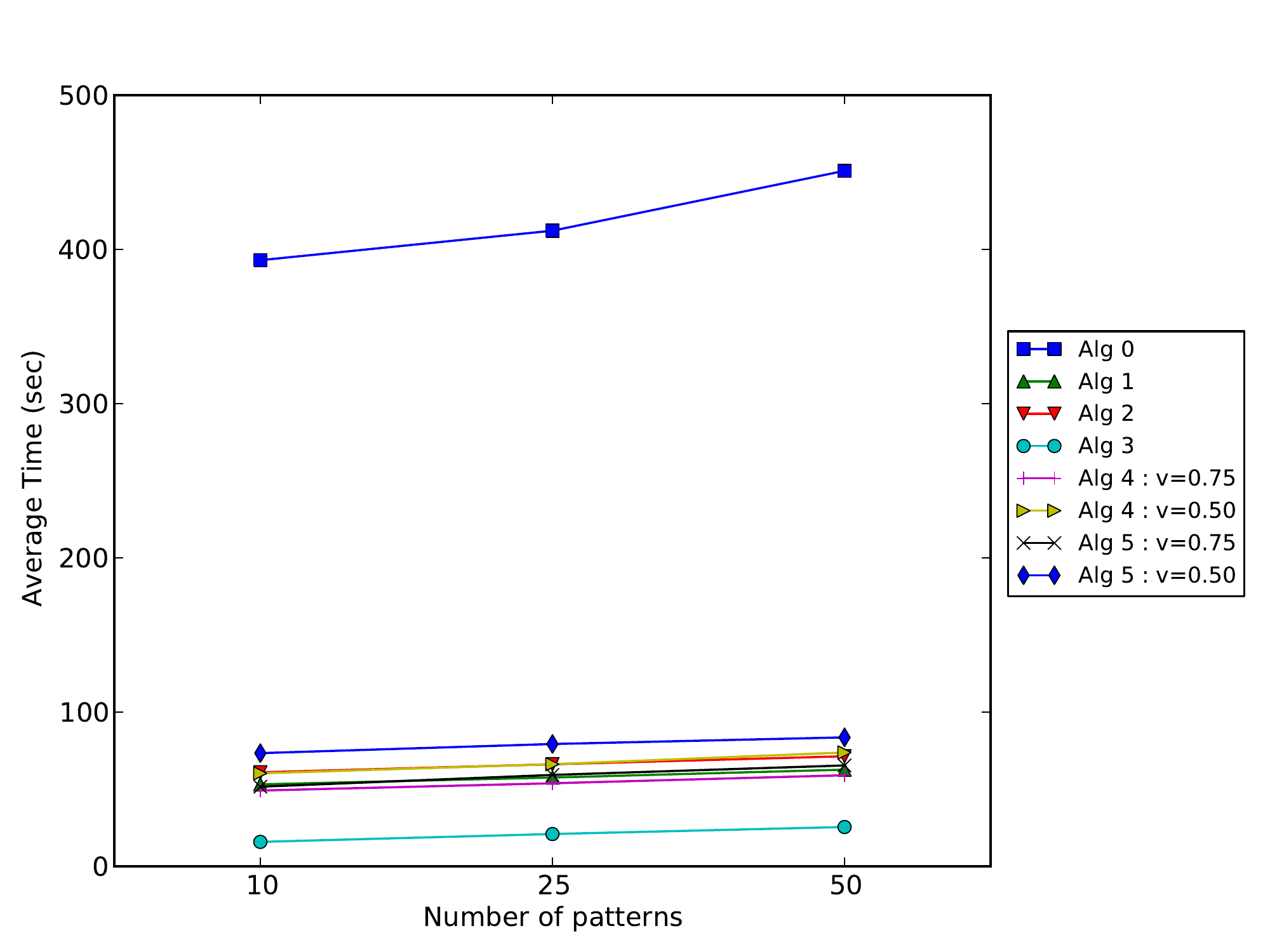}}
\subfigure[Memory] {\includegraphics[width=0.45\columnwidth,height=1.8in]{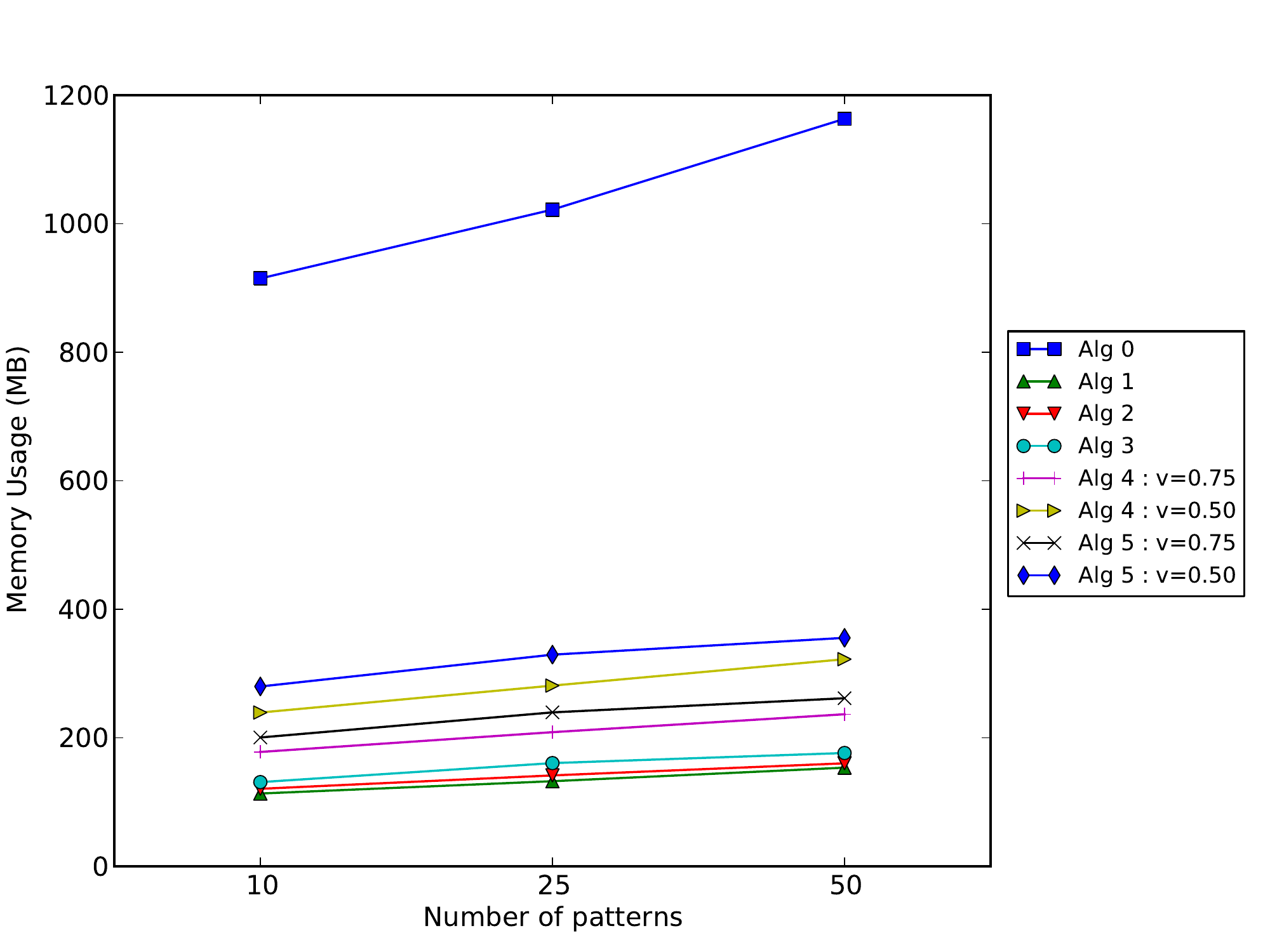}}
\caption{Effect of number of embedded patterns.}
\label{fig:patterns}
\end{figure}

In Fig.~\ref{fig:patterns}, we change the number of patterns embedded in the data and study the effect. The number of embedded patterns vary from 10 to 50. 
Once again the performance of our proposed methods is fairly flat in all the metrics. Alg 3, 4 and 5 are seen to be less sensitive to the number of patterns embedded in the data than Alg 1 and 2.

\subsubsection{Effect of Parameters}

So far the discussion has been about the effect of data characteristics of the synthetic data. The parameters of the mining algorithms were kept fixed. In this section we look at two important parameters of the algorithms, namely, the batch size $T_{b}$ and the number of batches that make up a window, $m$.

In Fig.~\ref{fig:batch}, the quality and performance metrics are plotted for three different batch sizes: $10^{3}$, $10^{4}$ and $10^{5}$ (in sec). Batch size appears to have a significant effect on precision and recall. There is a 10\% decrease in both precision and recall when batch size is reduced to $10^{3}$ sec from $10^{5}$. But note that a 100 fold decrease in batch size only changes quality of result by 10\%.

It is not hard to imagine that for smaller batch sizes the episode statistics can have higher variability in different batches resulting in a lower precision and recall over the window. As the size of the batch grows the top-$k$ in the batch starts to resemble the top-$k$ of the window. Transient patterns will not be able to gather sufficient support in a large batch size.

As expected, the runtimes and memory usage are directly proportional to the batch size in all cases. The extra space and time is required only to handle more data. Batch size does not play a role in growth of number of candidates in the mining process.

\begin{figure}[!ht]
\centering
\subfigure[Precision] {\includegraphics[width=0.45\columnwidth,height=1.8in]{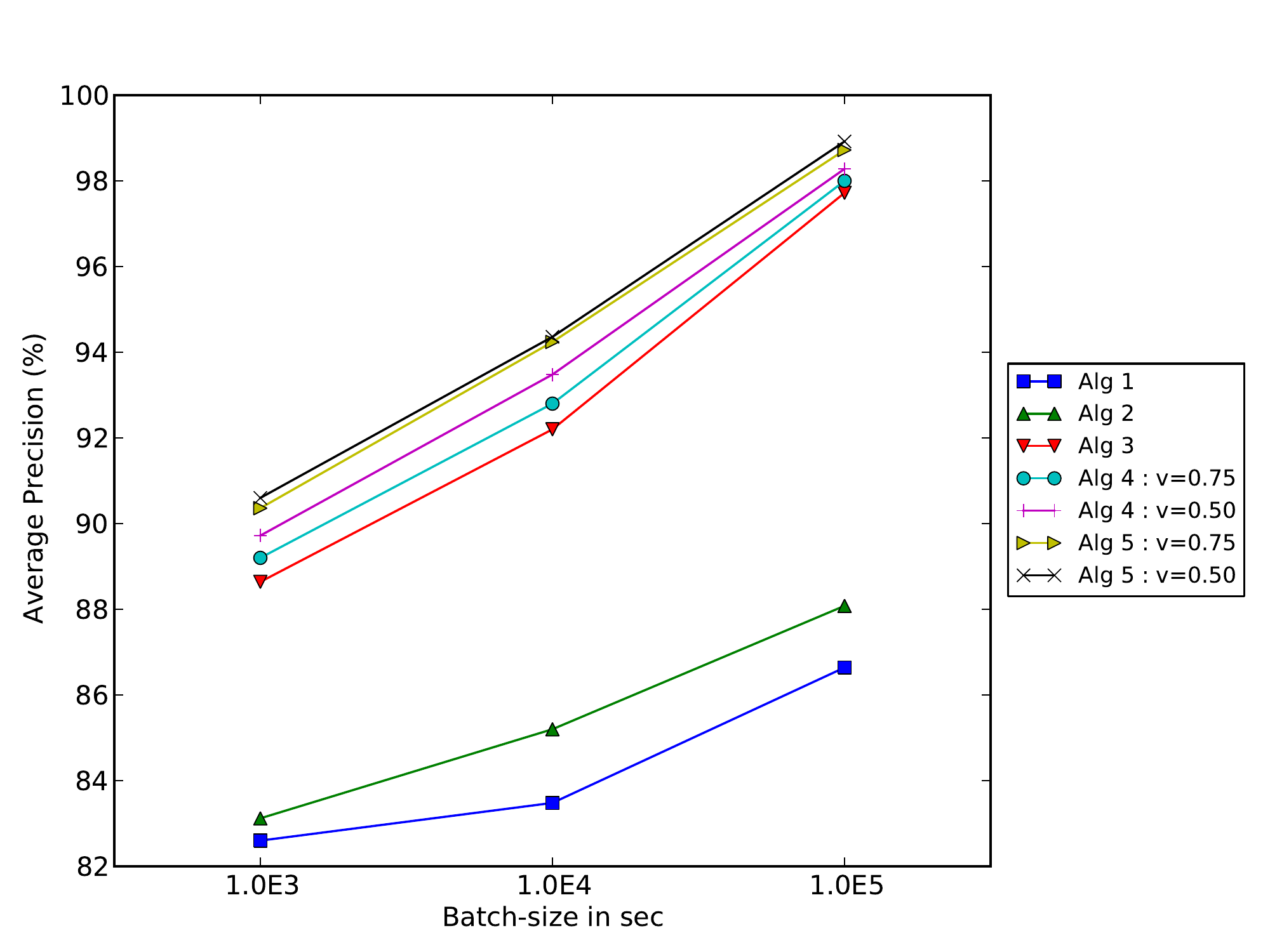}}
\subfigure[Recall] {\includegraphics[width=0.45\columnwidth,height=1.8in]{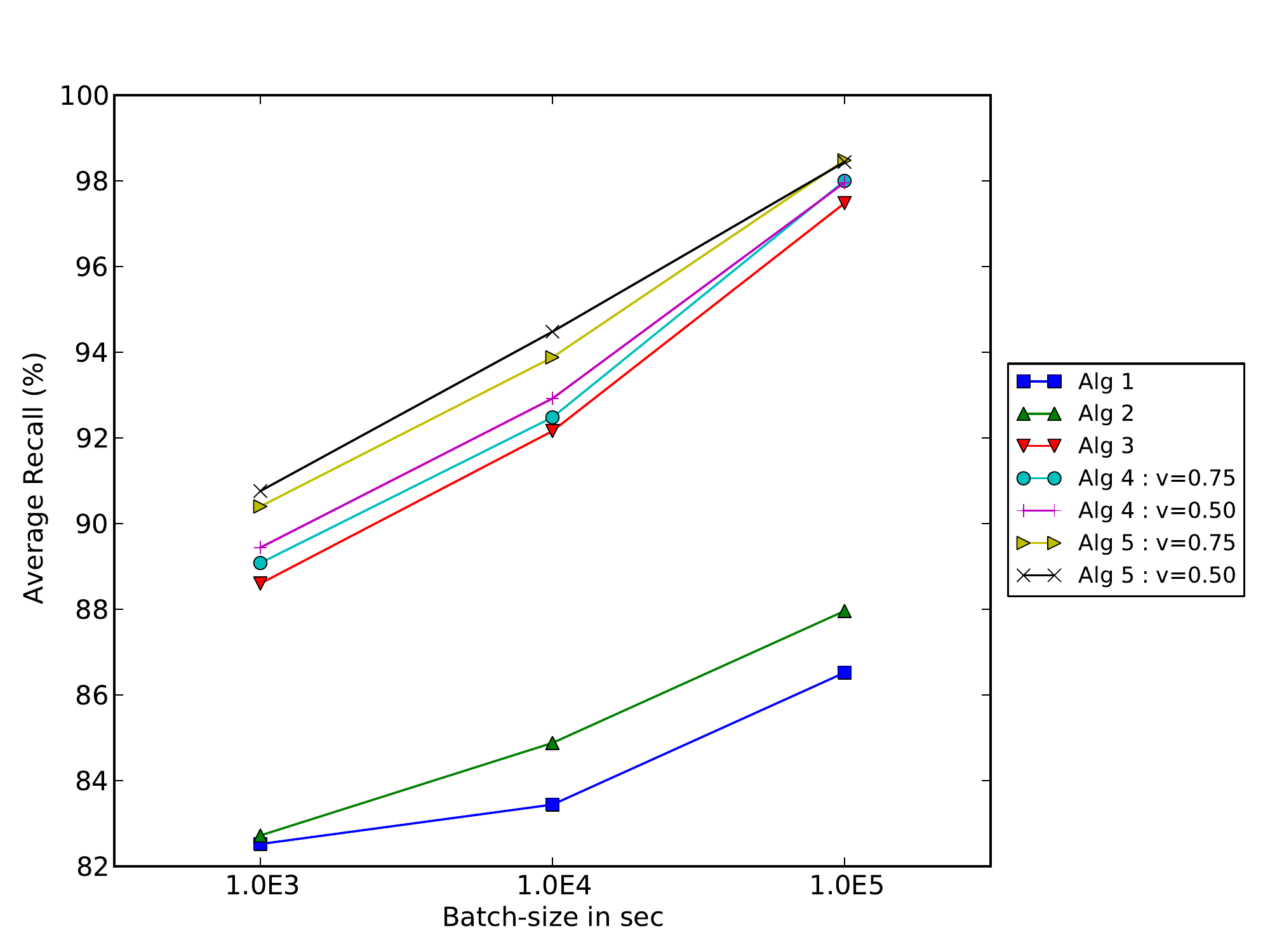}}
\subfigure[Runtime] {\includegraphics[width=0.45\columnwidth,height=1.8in]{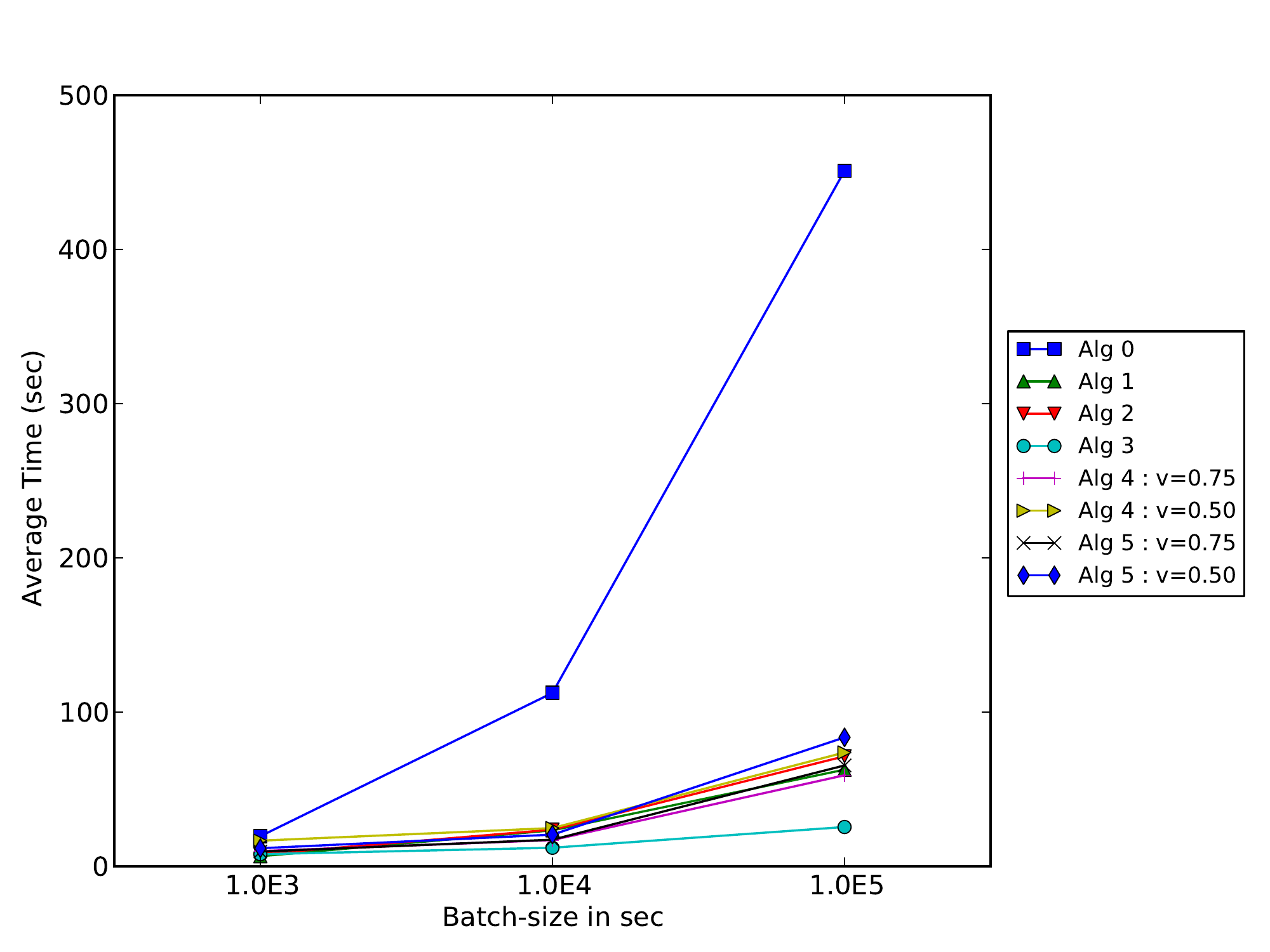}}
\subfigure[Memory] {\includegraphics[width=0.45\columnwidth,height=1.8in]{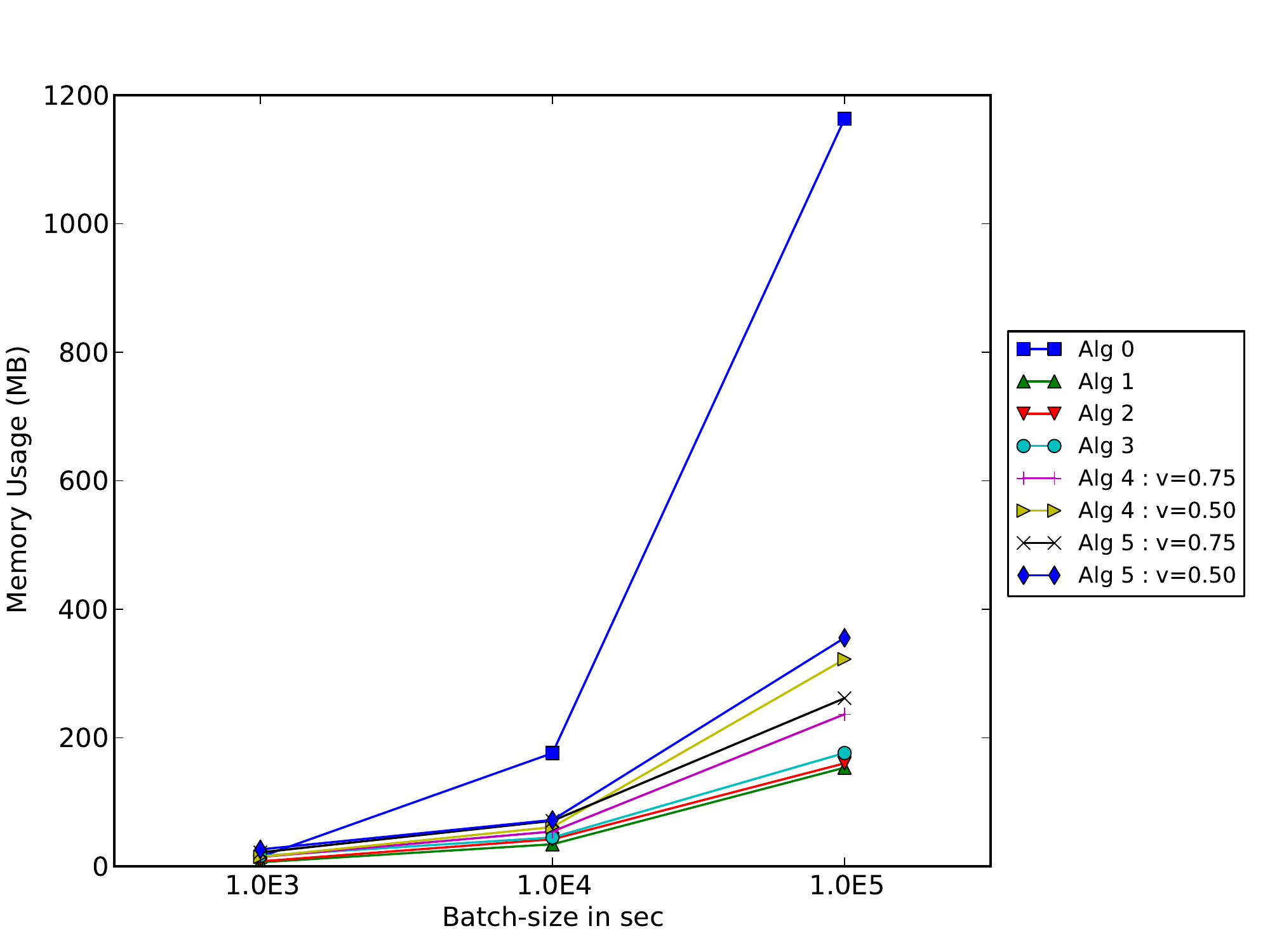}}
\caption{Effect of batchsize. Larger batch sizes have higher precision and recall. Precision and recall increase logarithmically with batch size. (Note that x-axis is in log scale)}
\label{fig:batch}
\end{figure}

Next in Fig.~\ref{fig:window}, we show how the number of windows in a batch affect the performance. Precision and recall are observed to decrease linearly with the number of batches in a window in Fig.~\ref{fig:window}(a) and (b), whereas the memory and runtime requirements grow linearly with the number of batches. The choice of number of batches provides the trade-off between the window size over which the user desires the frequent persistent patterns and the accuracy of the results. For larger window sizes the quality of the results will be poorer. Note that the memory usage and runtime does not increase much for algorithms other than Alg 0 (see Fig.~\ref{fig:window} (c) and (d)). Because these algorithms only process one batch of data irrespective of the number of batches in window. Although for larger windows the 
batch-wise support threshold decreases. In the synthetic datasets we see that this does not lead to unprecedented increase in the number of candidates.

\begin{figure}[!ht]
\centering
\subfigure[Precision] {\includegraphics[width=0.45\columnwidth,height=1.8in]{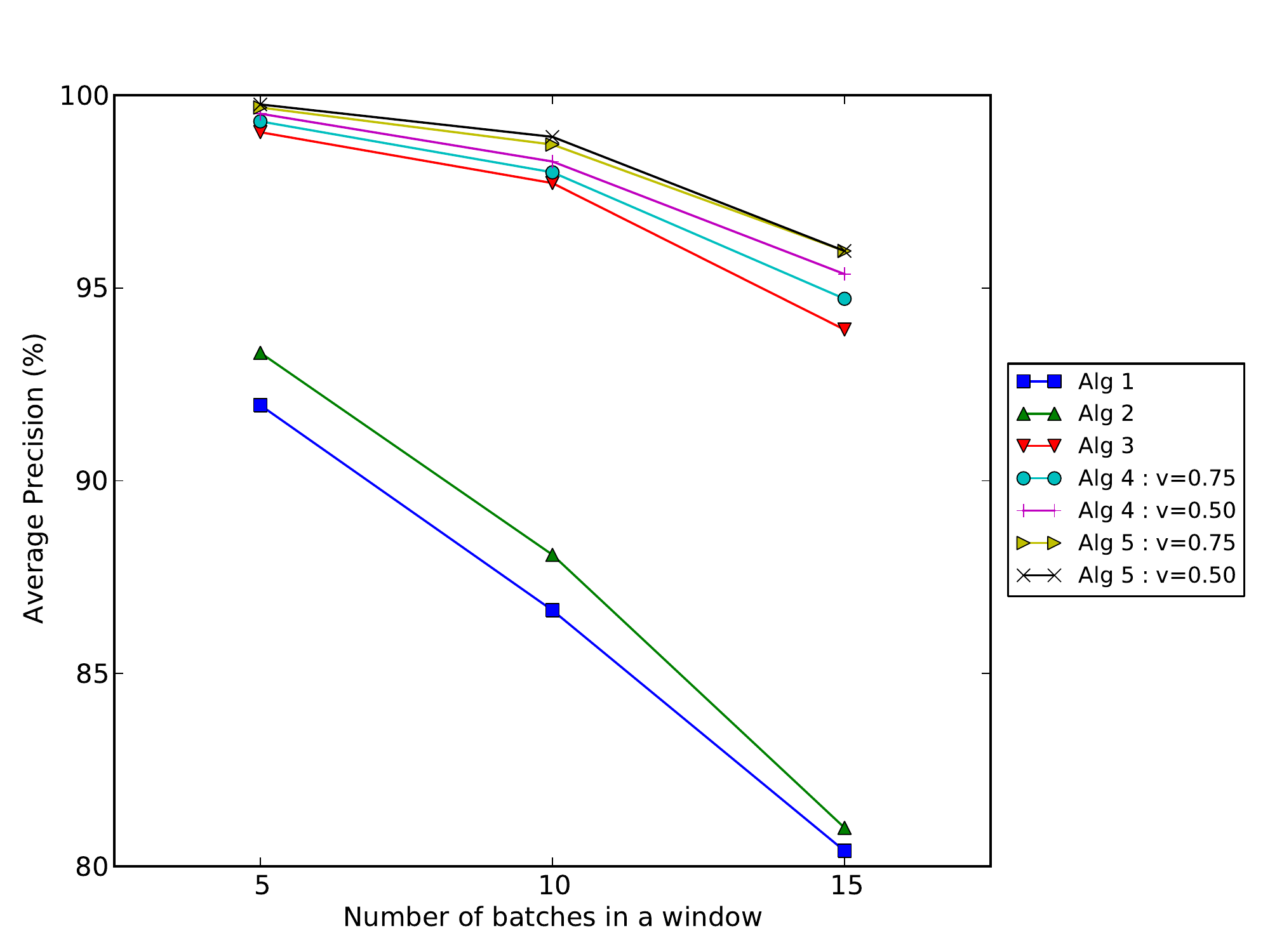}}
\subfigure[Recall] {\includegraphics[width=0.45\columnwidth,height=1.8in]{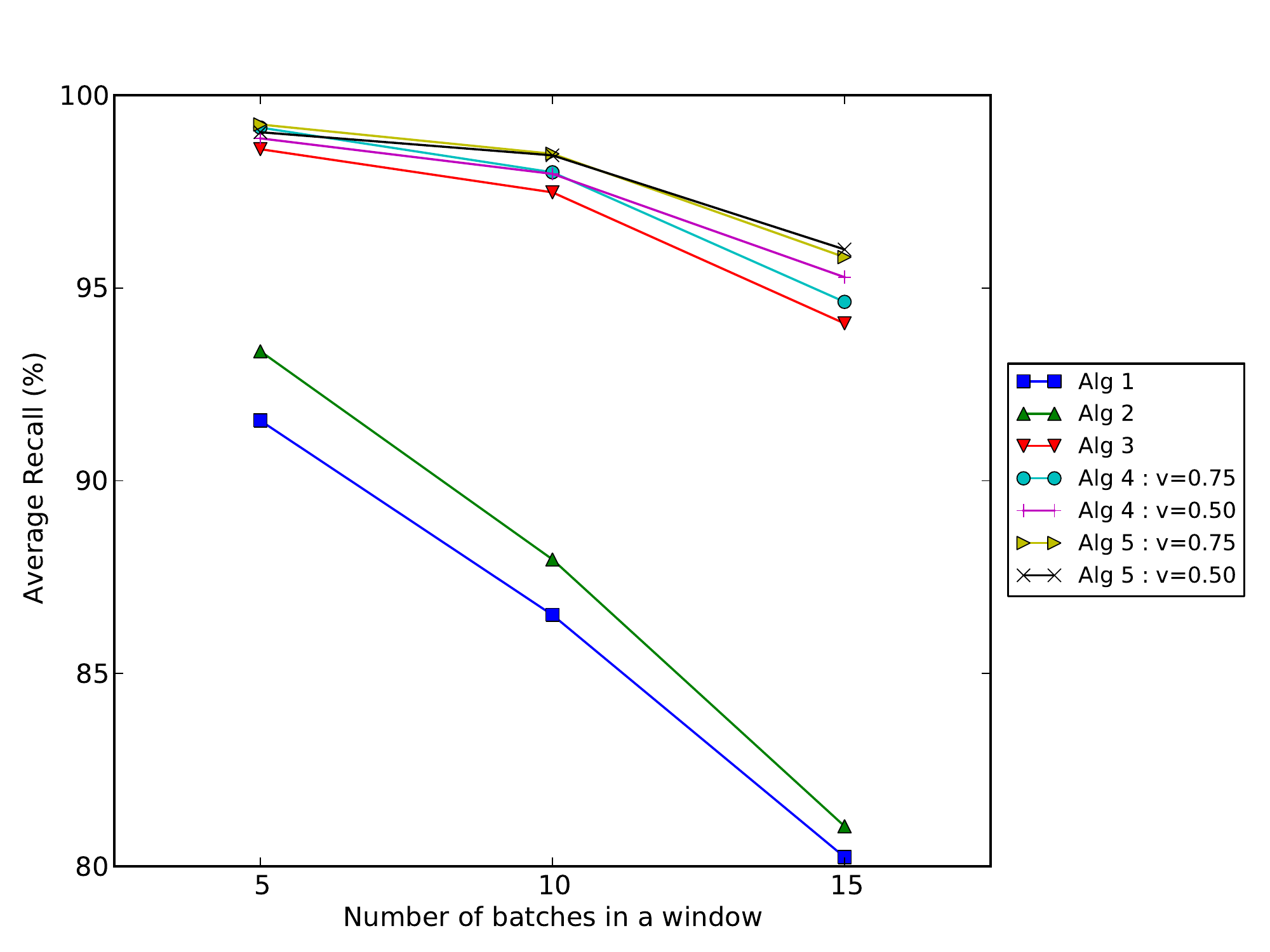}}
\subfigure[Runtime] {\includegraphics[width=0.45\columnwidth,height=1.8in]{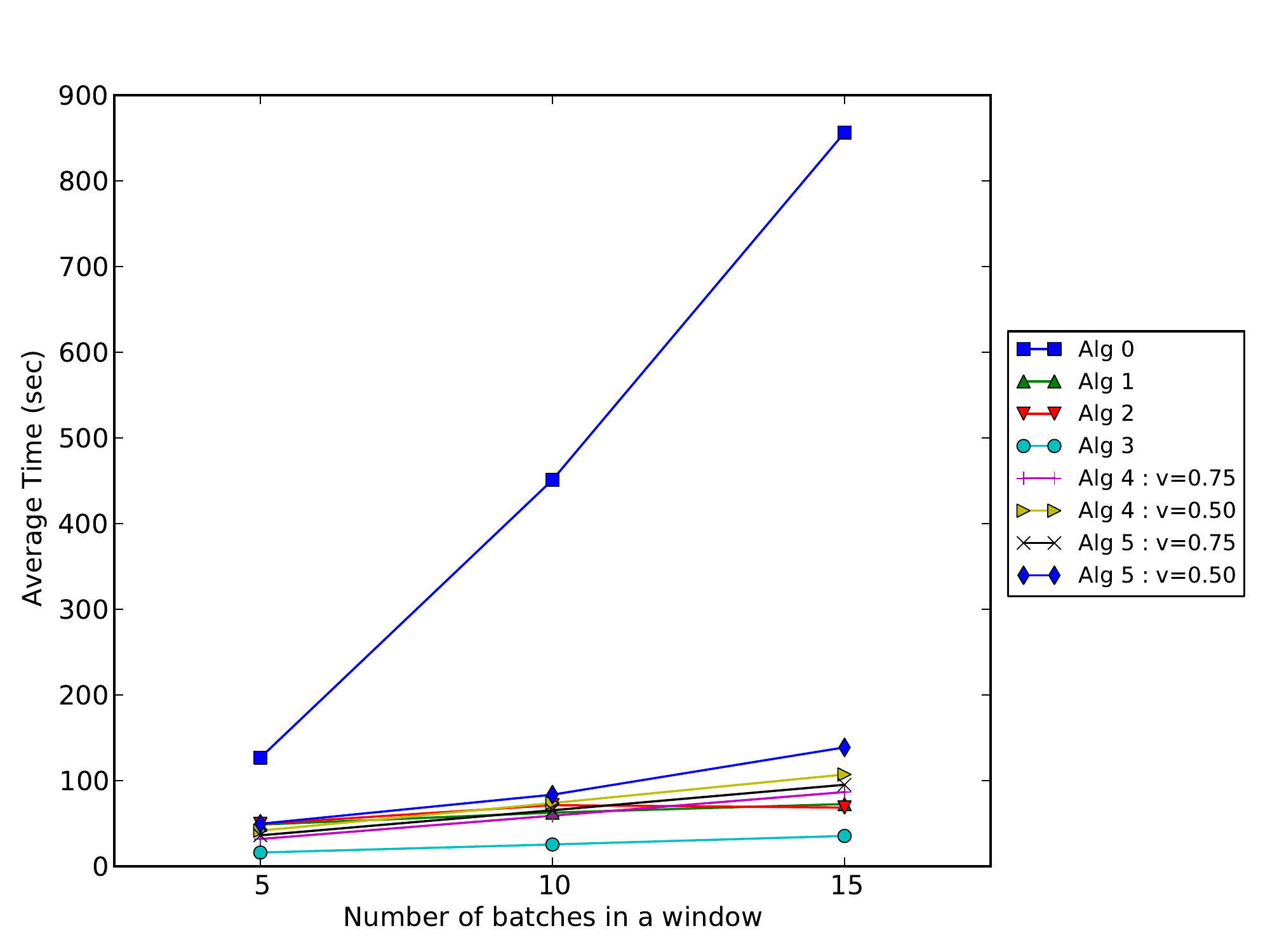}}
\subfigure[Memory] {\includegraphics[width=0.45\columnwidth,height=1.8in]{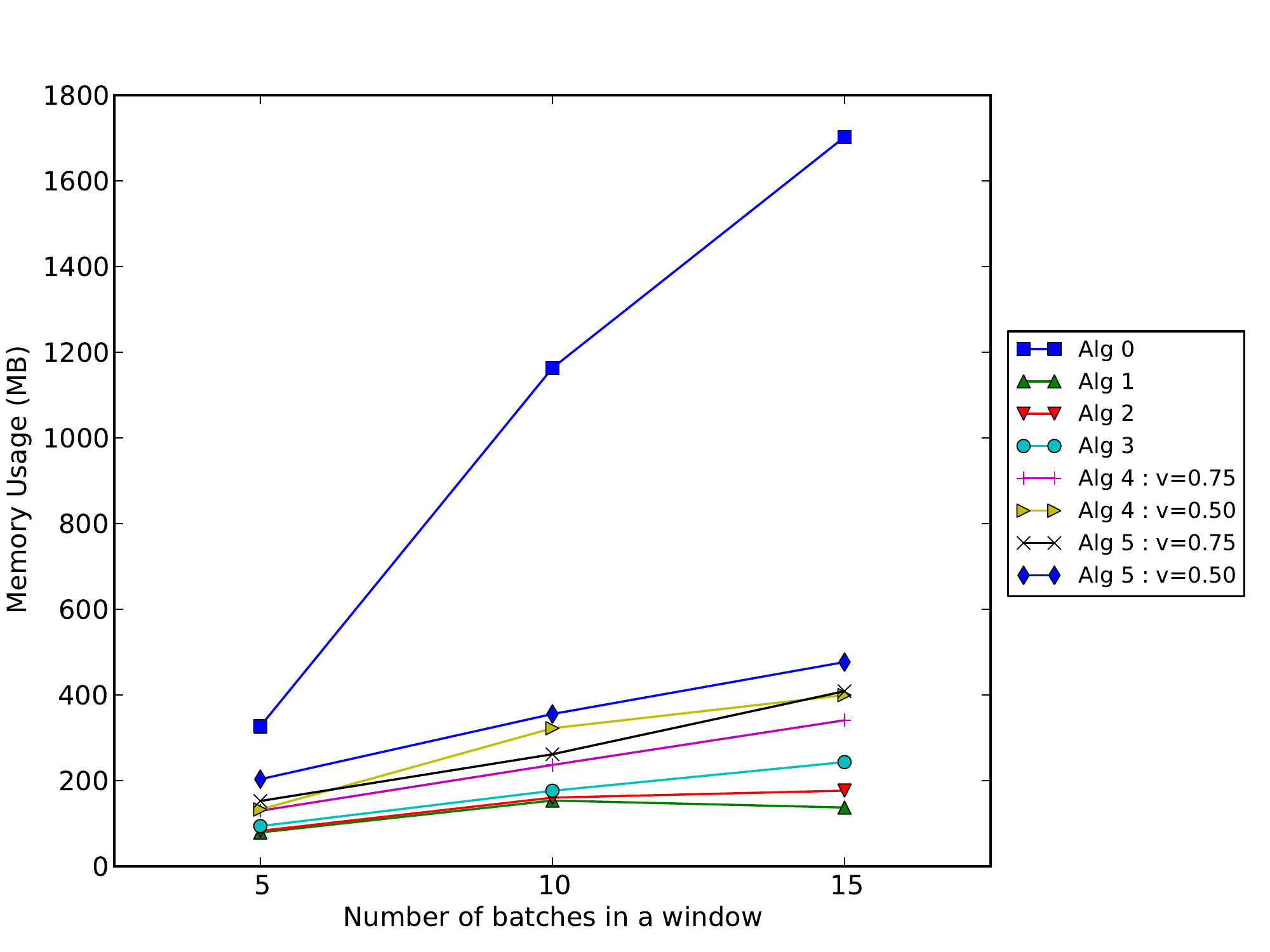}}
\caption{Effect of number of batches in a window. The memory usage and runtime does not increase much for algorithms other than Alg 0}
\label{fig:window}
\end{figure}

\subsection{Multi-neuronal Data}
Multi-electrode arrays provide high throughput recordings of the spiking activity in neuronal tissue and are hence rich sources of event data where events correspond to specific neurons being activated. We used the data from dissociated cortical cultures gathered by 
Steve Potter's laboratory at Georgia Tech \cite{Potter2006} over several days. This is a rich collection of recordings from a 64-electrode MEA setup.

We show the result of mining frequent episodes in the data collected over several days from Culture 6~\cite{Potter2006}. We use a batch size of 150 sec and all other parameters for mining are the same as that used for the synthetic data. The plots in Fig.~\ref{fig:real} show the performance of the different algorithms as time progresses. Alg 1 and 2 give very low precision values which implies that the top-$k$ in a batch is much different from the top-$k$ in the window. Alg 3, 4 and 5 perform equally well over the MEA data with Alg 3 giving the best runtime performance. 

\begin{figure}[!ht]
\centering
\subfigure[Precision] {\includegraphics[width=\columnwidth,height=1.4in]{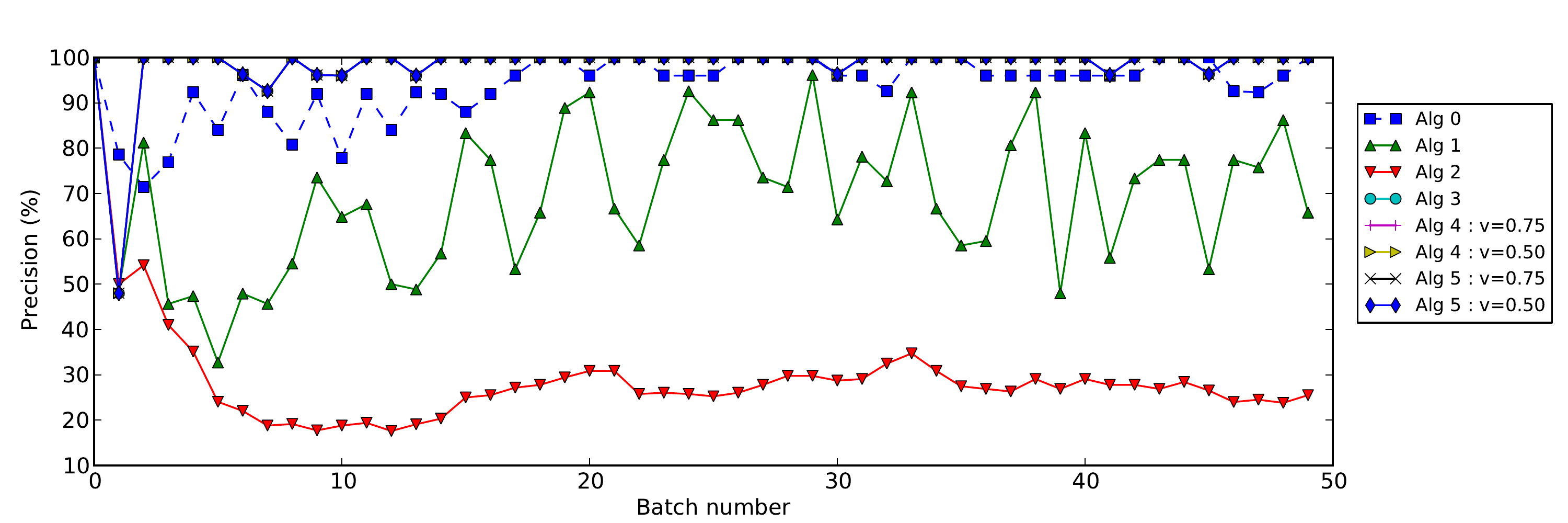}}
\subfigure[Recall] {\includegraphics[width=\columnwidth,height=1.4in]{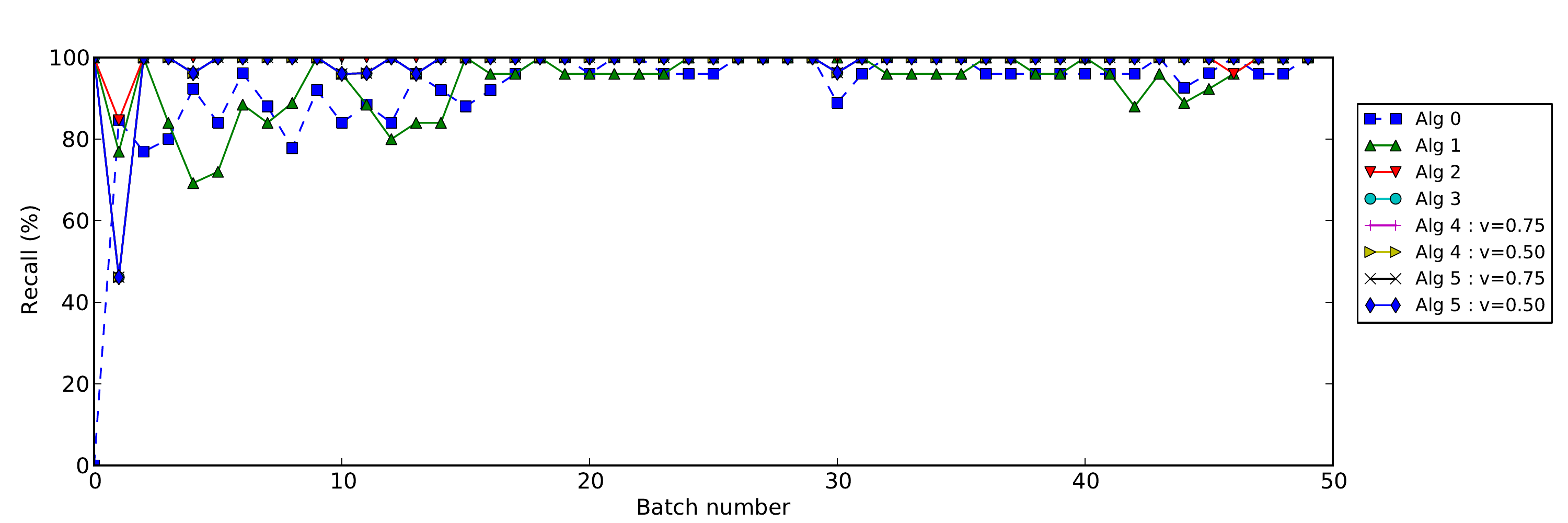}}
\subfigure[Runtime] {\includegraphics[width=\columnwidth,height=1.4in]{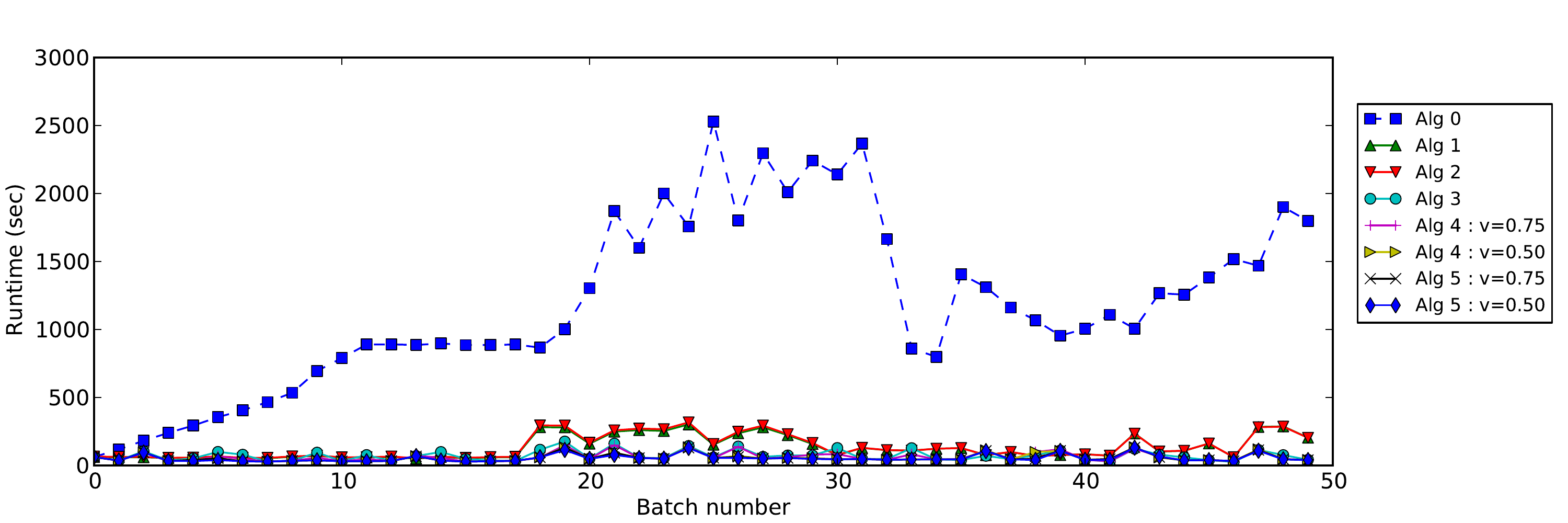}}
\subfigure[Memory] {\includegraphics[width=\columnwidth,height=1.4in]{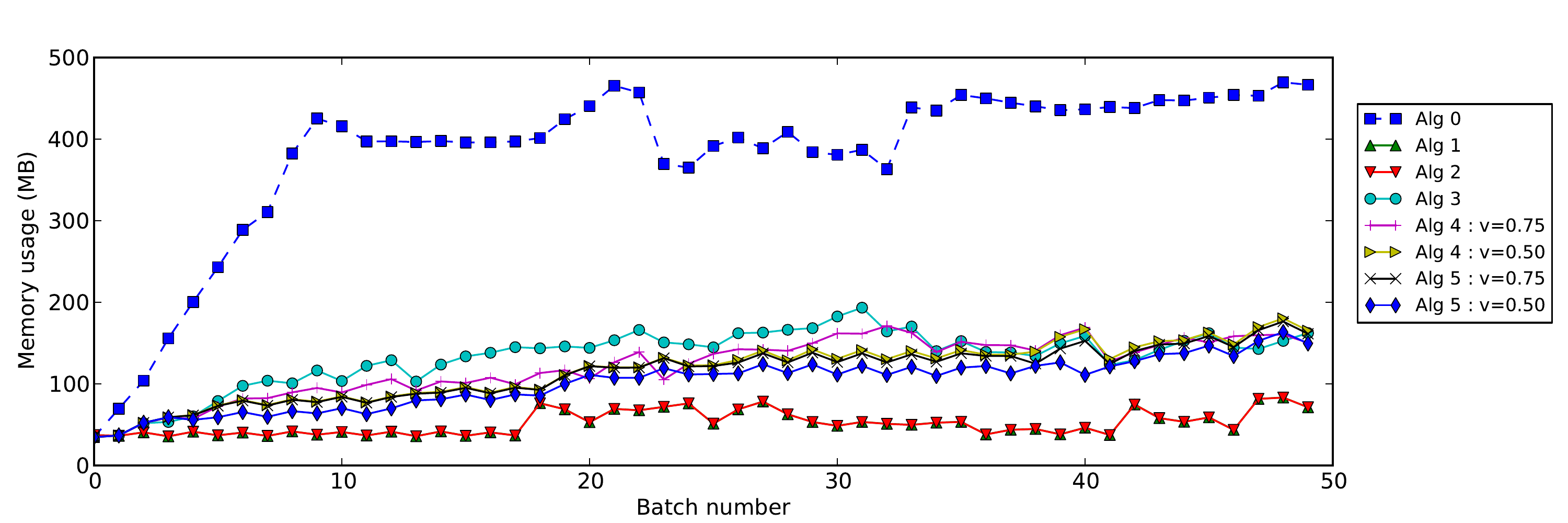}}
\caption{Comparison of performance of different algorithms on real multi-neuronal data.}
\label{fig:real}
\end{figure}

At times Alg 3 requires slightly higher memory than the other algorithm (Alg 4 and 5). This may seem counter intuitive as Alg 4 and 5 use lower frequency threshold. But since $\delta$ is dynamically estimated from all episodes being tracked by the algorithm it can easily be the case that the $\delta$ estimates made by Alg 3 are looser and hence result in higher memory usage.

%
%
%
%
%
%
%
%
%
%
%
%
%
%
%
%

\section{Related work}
\label{sec:related-work}

Most prior work in streaming pattern mining is related to frequent itemsets and sequential patterns~\cite{Karp:2003,Jin2005,Manku:2002,Calders2007}. Some interesting algorithms have also been proposed for streaming motif mining in time-series data~\cite{Mueen2010}. But these methods do not easily extend to other pattern classes like episodes and partial orders.  
To our knowledge, there has been very little work in the area of mining patterns in discrete event streams. In this section we discuss some of the existing methods for itemsets, sequential patterns, and motifs. 

Karp \textit{et al} proposed a one pass streaming algorithm for finding frequent events in an item sequence~\cite{Karp:2003}. The algorithm, at any given time, maintains a set $K$ of event types and their corresponding counts. Initially, this set is empty. When an event is read from the input sequence, if the event type exists in the set then its count is incremented. Otherwise the event type is inserted into the set $K$ with count 1. When the size of the set $K$ exceeds $\lfloor 1/\theta \rfloor$, the count of each event type in the set is decremented by 1 (and deleted from the set if count drops to zero). The key property is that any event type that occurs at least $n\theta$ times in the sequence is in the set $K$. Consider an event type that occurs $f$ times in the sequence, but is not in $K$. Each occurrence of this event type is eliminated together with more than $\lfloor 1/\theta \rfloor - 1$ occurrences of other event types achieved by decrementing all counts by 1. Thus, at least a total of $f/\theta$ elements are eliminated. Thus $f/\theta < n$, where $n$ is the number of events in the sequences and hence, $f < n\theta$. This method guarantees no false negatives for a given support threshold. But the space and time complexity of this algorithm varies inversely with the support threshold chosen by the user. This can be a problem when operating at low support thresholds. In \cite{Jin2005}, this approach was extended to mine frequent itemsets.

Lossy counting constitutes another important class of streaming algorithms proposed by Manku and Motwani in 2002 \cite{Manku:2002}. In this work an approximate counting algorithm for itemsets is described.
The algorithm stores a list of tuples which comprise an item or itemset, a lower bound on its count, and a maximum error term $(\Delta)$. When processing the $i^{th}$ item, if it is currently stored then its count is incremented by one; otherwise, a new tuple is created with the lower bound set to one, and $\Delta$ set to $\lfloor i \epsilon \rfloor$. Periodically, each tuple whose upper bound is less than $\lfloor i \epsilon \rfloor$ is deleted. This technique guarantees that the percentage error in reported counts is no more than $\epsilon$ and it is also shown that the space used by this algorithm is $O(\frac{1}{\epsilon}\log \epsilon n)$ for itemsets. Unfortunately, this method requires operating at very low support threshold $\epsilon$ in order to provide small enough error bounds.
In \cite{Mendes:2008}, the pattern growth algorithm - PrefixSpan \cite{PrefixSpan} for mining sequential patterns was extended to incorporate the idea of lossy counting.

In \cite{Calders2007}, the authors propose a new frequency measure for itemsets over data streams.  The frequency of an itemset in a stream is defined as its maximal frequency over all windows in the stream from any point in the past until the current time that satisfy a minimal length constraint. They present an incremental algorithm that produces the current frequencies of all frequent itemsets in the data stream. The focus of this work is on the new frequency measure and its unique properties.

In \cite{Mueen2010} an online algorithm for mining time series motifs was proposed. The algorithm uses an interesting data structure to find a pair of
approximately repeating subsequences in a window. The Euclidean distance measure
is used to measure the similarity of the motif sequences in the window. Unfortunately this notion does not extend naturally to discrete patterns. Further,
this motif mining formulation does not explicitly make use of a support or frequency threshold and returns exactly one pair of motifs that are found to be the closest in terms of distance.

A particular sore point in pattern mining is coming up with a frequency threshold for the mining process. Choice of this parameter is key to the success of any effective strategy for pruning the exponential search space of patterns. Mining the top-$k$ most frequent patterns has been proposed in the literature as a more intuitive formulation for the end user. In \cite{PLR10} we proposed an information theoretic principle for determining the frequency threshold that is ultimately used in learning a dynamic Bayesian network model for the data. In both cases the idea is to mine patterns at the highest possible support threshold  to either outputs the top-$k$ patterns or patterns that satisfy a minimum mutual information criteria. This is different from the approach adopted, for example, in lossy counting where the mining algorithm operates at support threshold proportional to the error bound. Therefore, in order to guarantee low errors, 
the algorithm tries to operate at the lowest possible threshold.

An episode or a general partial order pattern can be thought of as a generalization of itemsets where each item in the set is not confined to occur within the same transaction (i.e. at the same time tick) and there is additional structure in the form of ordering of events or items. In serial episodes, events must occur in exactly one particular order. Partial order patterns allow multiple orderings. In addition there could be repeated event types in an episode.  
The loosely coupled structure of events in an episode results in narrower separation between the frequencies of true and noisy patterns (i.e. resulting from random co-occurrences of events) and  quickly leads to combinatorial explosion of candidates when mining at low frequency thresholds.  Most of the itemset literature does not deal with the problem of candidate generation. The focus is on counting and not so much on efficient candidate generation schemes.
In this work we explore ways of doing both counting and candidate generation efficiently. Our goal is to devise algorithms that can operate at as high frequency thresholds as possible and yet 
give certain guarantees about the output patterns.

\section{Conclusions}
\label{sec:conclusions}

In this paper, we have studied the problem of
mining frequent episodes over changing data streams. In particular our contribution in this work is three fold. We unearth an interesting aspect of temporal data mining where the data owner may desire results over a span of time in the data that cannot fit in the memory or 
be processed at a rate faster than the data generation rate. We 
have proposed
a new sliding window model which slides forward in hops of batches. At any point only one batch of data is available for processing. We have studied
this problem and identified the theoretical guarantees one can give and the necessary assumptions for supporting them.

In many real applications we find the need for characterizing pattern not just based on their frequency but also their tendency to persist over time. In particular,
in neuroscience, the network structure underlying an ensemble of neurons changes much slowly in comparison to the culture wide periods bursting phenomenon. Thus separating the persistent patterns from the bursty ones can give us more insight into the underlying connectivity map of the network.
We have proposed the notion of
$(v,k)$ persistent patterns to address this problem and outlined methods to mine all $(v,k)$-persistent patterns in the data.
Finally we have provided
detailed experimental results on both synthetic and real data to show the advantages of the proposed methods.

Finally, we reiterate that
although we have focused on episodes, the ideas presented in this paper could
be applied to other pattern classes with similar considerations.

\bibliographystyle{abbrv}
\bibliography{bib/references}

\end{document}